\documentclass{amsart}
\usepackage{graphicx}
\usepackage{epstopdf}
\usepackage{color}
\usepackage{bbm, amssymb, amsmath, amsthm}

\usepackage{enumerate}

\newtheorem{theorem}{Theorem}
\newtheorem{corollary}[theorem]{Corollary}
\newtheorem{lemma}[theorem]{Lemma}

\newtheorem{definition}{Definition}
\newtheorem{algorithm}{Algorithm}
\newtheorem{example}{Example}

\newcommand{\ex}{\mathbb{E}}
\newcommand{\half}{\frac{1}{2}}

\newcommand{\sumt}{\sum_{t=1}^T}

\newcommand\numberthis{\addtocounter{equation}{1}\tag{\theequation}}
\newcommand{\prob}{\mathbb{P}}

\newcommand{\A}{\mathcal{A}}

\newcommand{\cost}{\mathrm{cost}}

\usepackage{ifthen}
\newboolean{showcomments}
\setboolean{showcomments}{true}
\newcommand{\adam}[1]{\ifthenelse{\boolean{showcomments}}
{ \textcolor{red}{(Adam says:  #1)}}{}}
\newcommand{\anish}[1]{\ifthenelse{\boolean{showcomments}}
{ \textcolor{red}{(Anish says:  #1)}}{}}
\newcommand{\niangjun}[1]{\ifthenelse{\boolean{showcomments}}
{ \textcolor{red}{(Niangjun says:  #1)}}{}}
\newcommand{\lachlan}[1]{\ifthenelse{\boolean{showcomments}}
{ \textcolor{red}{(Lachlan says:  #1)}}{}}
\newcommand{\addcites}[0]{\ifthenelse{\boolean{showcomments}}
{ \textcolor{green}{(add citation(s))}}{}}
\newcommand{\addref}[0]{\ifthenelse{\boolean{showcomments}}
{ \textcolor{green}{(add ref)}}{}}

\begin{document}
%
% --- Author Metadata here ---
%\conferenceinfo{SIGMETRICS}{'15 El Paso, Texas USA}
%\CopyrightYear{2015} % Allows default copyright year (20XX) to be over-ridden - IF NEED BE.
%\crdata{0-12345-67-8/90/01}  % Allows default copyright data (0-89791-88-6/97/05) to be over-ridden - IF NEED BE.
% --- End of Author Metadata ---

\title{Online Convex Optimization Using Predictions}

%\numberofauthors{1}

% \author{
% \alignauthor
% Niangjun Chen\\
%        \affaddr{California Institute of Technology}\\
%        \affaddr{1200 E California Blvd}\\
%        \affaddr{California, USA}\\
%        \email{ncchen@caltech.edu}
% % 2nd. author
% \alignauthor
% Anish Agarwal\\
%        \affaddr{California Institute of Technology}\\
%        \affaddr{1200 E California Blvd}\\
%        \affaddr{California, USA}\\
%        \email{aagarwal@caltech.edu}
% % 3rd. author
% \alignauthor Adam Wierman\\
%        \affaddr{California Institute of Technology}\\
%        \affaddr{1200 E California Blvd}\\
%        \affaddr{California, USA}\\
%        \email{adamw@caltech.edu}
% \and  % use '\and' if you need 'another row' of author names
% % 4th. author
% \alignauthor Siddharth Barman\\
%        \affaddr{California Institute of Technology}\\
%        \affaddr{1200 E California Blvd}\\
%        \affaddr{California, USA}\\
%        \email{sid.barman@caltech.edu}
% % 5th. author
% \alignauthor Lachlan L. H. Andrew\\
%        \affaddr{Monash University}\\
%        \affaddr{Victoria 3800}\\
%        \affaddr{Australia}\\
%        \email{lachlan.andrew@monash.edu}
% }
\author{Niangjun Chen, Anish Agarwal, Adam Wierman,  Siddharth Barman, Lachlan L. H. Andrew}
%\address{$^\dagger$California Institute of Technology, Pasadena, CA, USA} 
%	\email{ \{ncchen,aagarwal,adamw,sid.barman\} @caltech.edu}
	%\email{\{ncchen,aagarwal,adamw,sid.barman\} @caltech.edu}\\
	%\affaddr{$^\ddagger$Monash University, Australia, E-mail: lachlan.andrew@monash.edu}\\
	%\email{lachlan.andrew@monash.edu}
% \alignauthor Lachlan L. H. Andrew\\
%        \affaddr{Monash University}\\
%        \affaddr{Victoria 3800}\\
%        \affaddr{Australia}\\
%        \email{lachlan.andrew@monash.edu}

\maketitle
\begin{abstract}
Making use of predictions is a crucial, but under-explored, area of online algorithms. This paper studies a class of online optimization problems where we have external noisy predictions available.  We propose a stochastic prediction error model that generalizes prior models in the learning and stochastic control communities, incorporates correlation among prediction errors, and captures the fact that predictions improve as time passes.   We prove that achieving sublinear regret and constant competitive ratio for online algorithms requires the use of an unbounded prediction window in adversarial settings, but that under more realistic stochastic prediction error models it is possible to use Averaging Fixed Horizon Control (AFHC) to simultaneously achieve sublinear regret and constant competitive ratio in expectation using only a constant-sized prediction window. Furthermore, we show that the performance of AFHC is tightly concentrated around its mean. 
% A category with the (minimum) three required fields
%\category{F.2.0}{Analysis of Algorithms and Problem Complexity}{General}
%A category including the fourth, optional field follows...
%\category{D.2.8}{Software Engineering}{Metrics}[complexity measures, performance measures]
%\terms{Algorithms, Performance, Theory}
%\keywords{Online Convex Optimization, Prediction, Regret, Competitive Ratio}
%
% \textbf{Catogories and Subject Descriptors:} F.2.0 [Analysis of Algorithms and Problem Complexity]: General
%
% \textbf{Keywords:} Online Convex Optimization; Prediction; Regret; Competitive Ratio
\end{abstract}

\section{Introduction}
\label{sec:intro}

Making use of predictions about the future is a crucial, but under-explored, area of online algorithms.  In this paper, we use online convex optimization to illustrate the insights that can be gained from incorporating a general, realistic model of prediction noise into the analysis of online algorithms.

\textbf{Online convex optimization.} In an online convex optimization (OCO) problem, a learner interacts with an environment in a sequence of rounds.  In round $t$ the learner chooses an action $x_t$ from a convex decision/action space $G$, and then the environment reveals a convex cost function $c_t$ and the learner pays cost $c_t(x_t)$. An algorithm's goal is to minimize total cost over a (long) horizon $T$.

OCO has a rich theory and a wide range of important applications. In computer science, it is most associated with the so-called $k$-experts problem, an online learning problem where in each round $t$ the algorithm chooses one of $k$ possible actions, viewed as following the advice of one of $k$ ``experts''.

However, OCO is being increasingly broadly applied and, recently has become
prominent in networking and cloud computing applications, including the design
of dynamic capacity planning, load shifting and demand response for data
centers \cite{kim2014real, lin2013, liu2014, liu2013, narayanaswamy2012online},
geographical load balancing of internet-scale systems \cite{lin2012,
wang2014exploring}, electrical vehicle charging \cite{gan2013optimal,
kim2014real}, video streaming~\cite{joseph2011variability,joseph2012jointly}
and thermal management of systems-on-chip \cite{zanini2009multicore,zanini2010online}.

In typical applications of online convex optimization in networking and cloud computing there is an additional cost in each round, termed a ``switching cost'', that captures the cost of changing actions during a round. Specifically, the cost is $c_t(x_t)+\parallel x_t - x_{t-1} \parallel$, where $\parallel \cdot \parallel$ is a norm (often the one-norm).  This additional term makes the online problem more challenging since the optimal choice in a round then depends on future cost functions.  These ``smoothed'' online convex optimization problems have received considerable attention in the context of networking and cloud computing applications, e.g., \cite{lin2012,lin2013,liu2014,liu2013,lu2013}, and are also relevant for many more traditional online convex optimization applications where, in reality, there is a cost associated with a change in action, e.g., portfolio management.  \emph{We focus on smoothed online convex optimization problems.}

\textbf{A mismatch between theory and practice.} As OCO algorithms make their way into networking and cloud computing applications, it is increasingly clear that there is a mismatch between the pessimistic results provided by the theoretical analysis (which is typically adversarial) and the near-optimal performance observed in practice.

Concretely, two main performance metrics have been studied in the literature: \emph{regret}, defined as the difference between the cost of the algorithm and the cost of the offline optimal static solution, and the \emph{competitive ratio}, defined as the maximum ratio between the cost of the algorithm and the cost of the offline optimal (dynamic) solution.

Within the \emph{machine learning community}, regret has been heavily studied \cite{hazan2007, xiao2010, zinkevich2003} and there are many simple algorithms that provide provably sublinear regret (also called ``no regret'').  For example, online gradient descent achieves $O(\sqrt{T})$-regret \cite{zinkevich2003}, even when there are switching costs \cite{andrew2013}.  In contrast, the \emph{online algorithms community} considers a more general class of problems called ``metrical task systems'' (MTS) and focuses on competitive ratio \cite{blum1992, borodin1992, lin2013}.  Most results in this literature are ``negative'', e.g., when $c_t$ are arbitrary, the competitive ratio grows without bound as the number of states in the decision space grows \cite{borodin1992}.  Exceptions to such negative results come only when structure is imposed on either the cost functions or the decision space, e.g.,  when the decision space is  \emph{one-dimensional} it is possible for an online algorithm to have a constant competitive ratio, e.g., \cite{lin2013}.  However, \emph{even in this simple setting no algorithms performs well for both competitive ratio and regret}.  No online algorithm can have sublinear regret and a constant competitive ratio, even if the decision space is one-dimensional and cost functions are linear \cite{andrew2013}.

In contrast to the pessimism of the analytic work, applications in networking and cloud computing have shown that OCO algorithms can significantly outperform the static optimum while nearly matching the performance of the dynamic optimal, i.e., simultaneously do well for regret and the competitive ratio.  Examples include dynamic capacity management of a single data center \cite{amur2010robust, lin2013} and geographical load balancing across multiple data centers \cite{lin2012, narayanaswamy2012online, qureshi2009cutting}.

It is tempting to attribute this discrepancy to the fact that practical workloads are not adversarial.  However, a more important factor is that, in reality, \emph{algorithms can exploit relatively accurate predictions about the future}, such as diurnal variations \cite{arlitt1996, gmach2007, liu2013}.  But a more important contrast between the theory and application is simply that, in reality, \emph{predictions about the future are available and accurate, and thus play a crucial role in the algorithms}.

\textbf{Incorporating predictions.} It is no surprise that predictions are crucial to online algorithms in practice. In OCO, knowledge about future cost functions is valuable, even when noisy.  However, despite the importance of predictions, we do not understand how prediction noise affects the performance (and design) of online convex optimization algorithms.

This is not due to a lack of effort. Most papers that apply OCO algorithms to networking and cloud computing applications study the impact of prediction noise, e.g., \cite{adnan2012energy,ananthanarayanan2013,liu2013,narayanaswamy2012online}.  Typically, these consider numerical simulations where i.i.d.\ noise terms with different levels of variability are added to the parameter being predicted, e.g., \cite{gmach2010capacity, thereska2009sierra}.  While this is a valuable first step, it does not provide any \emph{guarantees} about the performance of the algorithm with realistic prediction errors (which tend to be correlated, since an overestimate in one period is likely followed by another overestimate) and further does not help inform the \emph{design} of algorithms that can effectively use predictions.

Though most work on predictions has been simulation based, there has also been significant work done seeking analytic guarantees.  This literature can be categorized into:

\begin{enumerate}[(i)]
\item \emph{Worst-case models} of prediction error typically assume that there exists a lookahead window $\omega$ such that within that window, prediction is near-perfect (\emph{too optimistic}), and outside that window the workload is adversarial (\emph{too pessimistic}), e.g., \cite{borodin1992, lin2013, lin2012, lu2013, camacho2014balance}.

\item \emph{Simple stochastic models} of prediction error typically consider i.i.d. errors, e.g., \cite{boyd2012,liu2013,liu2014}.  Although this is analytically appealing, it ignores important features of prediction errors, as described in the next section.

\item \emph{Detailed stochastic models} of specific predictors applied for specific signal models, such as \cite{zhou2000,sage1971estimation,sastry2011adaptive,kailath2000}.  This leads to less pessimistic results, but the guarantees, and the algorithms themselves, become \emph{too fragile} to assumptions on the system evolution.
\end{enumerate}

\textbf{Contributions of this paper.}
First, \emph{we introduce a general colored noise model for studying prediction errors in online convex optimization problems.} The model captures three important features of real predictors:  (i) it allows for arbitrary correlations in prediction errors (e.g., both short and long range correlations); (ii) the quality of predictions decreases the further in the future we try to look ahead; and (iii) predictions about the future are updated as time passes. Further, it strikes a middle ground between the worst-case and stochastic approaches.  In particular, it does not make any assumptions about an underlying stochastic process or the design of the predictor.  Instead, it only makes (weak) assumptions about the stochastic form of the error of the predictor; these assumptions are satisfied by many common stochastic models, e.g., the prediction error of standard Weiner filters \cite{wiener1949} and Kalman filters \cite{kalman1960}.  Importantly, by being agnostic to the underlying stochastic process, the model allows worst-case analysis with respect to the realization of the underlying cost functions.
%\anish{would remove - but guarantees that the predictor is well-adapted, i.e., the prediction error is well-behaved.}
%, (weighted) moving average predictors \addcites,

Second, using this model, \emph{we show that a simple algorithm, Averaging Fixed Horizon Control (AFHC) \cite{lin2012}, simultaneously achieves sublinear regret and a constant competitive ratio in expectation using very limited prediction}, i.e., a prediction window of size $O(1)$, in nearly all situations when it is feasible for an online algorithm to do so (Theorem \ref{thm: general-simultaneous-afhc}).  Further, we show that the performance of AFHC is tightly concentrated around its mean (Theorem~\ref{thm: afhc-concentration}). Thus, AFHC extracts the asymptotically optimal value from predictions. Additionally, our results inform the choice of the optimal prediction window size.  (For ease of presentation, both  Theorems \ref{thm: regret-afhc} and \ref{thm: afhc-concentration} are stated and proven for the specific case of online LASSO -- see Section \ref{sec:oco} -- but the proof technique can be generalized in a straightforward way.)

Importantly, Theorem \ref{thm: cd-afhc} highlights that the dominant factor impacting whether the prediction window should be long or short in AFHC is not the variance of the noise, but rather the correlation structure of the noise.  For example, if prediction errors are i.i.d.\ then it is optimal for AFHC to look ahead as far as possible (i.e., $T$) regardless of the variance, but if prediction errors have strong short-range dependencies then the optimal prediction window is constant sized regardless of the variance.

Previously, AFHC had only been analyzed in the adversarial model \cite{lin2013}, and our results are in stark contrast to the pessimism of prior work. To highlight this, we prove that in the ``easiest'' adversarial model (where predictions are exact within the prediction window), no online algorithm can  achieve sublinear regret \emph{and} a constant competitive ratio when using a prediction window of constant size (Theorem \ref{thm: worstcaselower}).  This contrast emphasizes the value of moving to a more realistic stochastic model of prediction error.
%simultaneously

\section{Online Convex Optimization \\ with Switching Costs}
\label{sec:oco}

Throughout this paper we consider online convex optimization problems with switching costs, i.e., ``smoothed'' online convex optimization (SOCO) problems.

\subsection{Problem Formulation}
%\lachlan{I think we should cut these two paragraphs substantially. They formulate a problem more general than the one we solve, and obscure the actual model.  For example, the use of $F$ conflicts with the notation $F(w)$.} \adam{I think it's useful to have the general setting in order to contrast our setting with.  Also, sense we talk a lot about prior results and impossibility in the general setting, it's good to define it properly. Also, we claim we look at LASSO as a special case for ease of presentation, so it's good to have the special case too. }
The standard formulation of an online optimization problem with switching costs considers a convex decision/action space $G \subset \mathbb{R}^n$ and a sequence of cost functions $\{c_1, c_2, \ldots\}$, where each $c_t: G \rightarrow \mathbb{R}^+$ is convex. At time $t$, the online algorithm first chooses an action, which is a vector $x_t \in G$, the environment chooses a cost function $c_t$ from a set $\mathcal{C}$, and the algorithm pays a stage cost $c_t(x_t)$ and a switching cost $\beta||x_t - x_{t-1}||$ where $\beta \in (\mathbb{R}^+)$.  Thus, the total cost of the online algorithm is defined to be
\begin{equation}
\cost(ALG) = \ex_x\left[\sumt c_t(x_t) + \beta||x_t - x_{t-1}|| \right],
\label{eqn: cost_general}
\end{equation}
where $x_1, \ldots, x_T$ are the actions chosen by the algorithm, $ALG$. Without loss of generality, assume the initial action $x_0 = 0$, the expectation is over any randomness used by the algorithm, and $||\cdot||$ is a seminorm on $\mathbb{R}^n$.

Typically, a number of assumptions about the action space $G$ and the cost functions $c_t$ are made to allow positive results to be derived.  In particular, the action set $G$ is often assumed to be {closed}, {nonempty}, and {bounded}, where by bounded we mean that there exists $D \in \mathbb{R}$ such that for all $x, y \in G, ||x - y|| \le D. $  Further, the cost functions $c_t$ are assumed to have a {uniformly bounded subgradient, i.e., there exists $N \in \mathbb{R^+}$ such that, for all $x \in G,$ $||\nabla c_t(x)|| \le N.$

%Assumption 2 and 3 are prevalent in the regret minimization community, where the $D$ and $N$ are constants that are present in the regret bounds for typical regret minimizing algorithms. In practice, the performance bound may be undesirable if the feasible $F$ set is large, or the cost function $c_t$ has large subgradient for some values of $x \in F$.

%Also, an online algorithm with lookahead $(w+1)$ is able to see the cost
%function $c_t, c_{t+1} \ldots, c_{t+w}$ at time $t$. Furthermore, it is assumed that the lookahead is {\bf exact}. This model is too optimistic about prediction within the lookahead window, but too pessimistic outside the lookahead window, where it is assumed that we know nothing beyond. It would be desirable to study performance of online algorithms where prediction is noisy.

Since our focus in this paper is on predictions, we consider a variation of the above with \emph{parameterized} cost functions $c_t(x_t;y_t)$, where the parameter $y_t$ is the focus of prediction.  Further, except when considering worst-case predictions, we adopt a specific form of $c_t$ for concreteness. We focus on a tracking problem where the online algorithm is trying to do a ``smooth'' tracking of $y_t$ and pays a least square penalty each round.  
\begin{equation}
\cost(ALG) = \ex_x \left[\sumt \half ||y_t - K x_t ||^2_2 + \beta ||x_{t} - x_{t-1} ||_1. \right],
\label{eqn: cost_tracking}
\end{equation}
where the target $y_t \in \mathbb{R}^m$, and $K \in \mathbb{R}^{m \times n}$ is a (known) linear map that transforms the control variable into the space of the tracking target. Let $K^\dagger$ be the Moore-Penrose pseudoinverse of $K$.

We focus on this form because it represents an online version of the LASSO (Least Absolute Shrinkage and Selection Operator) formulation, which is widely studied in a variety of contexts, e.g., see \cite{candes2009, chandrasekaran2012, figueiredo2007, tibshirani1996} and the references therein. Typically in LASSO the one-norm regularizer is used to induce sparsity in the solution.  In our case, this corresponds to specifying that a good solution does not change too much, i.e., $x_t - x_{t-1} \ne 0$ is infrequent. Importantly, the focus on LASSO, i.e., the two-norm loss function and one-norm regularizer, is simply for concreteness and ease of presentation.  Our proof technique generalizes (at the expense of length and complexity). 

%and, in fact, one can expect the one-norm regularizer to be the ``hardest'' case for an online algorithm to perform well.

We assume that $K^TK$ is invertible and that the static
optimal solution to \eqref{eqn: cost_tracking} is positive. Neither of these
is particularly restrictive.  If $K$ has full column rank then $K^TK$ is invertible. This is a reasonable, for example, when the dimensionality of the action space $G$ is small relative to the output space. Note that typically $K$ is designed, and so it can be chosen to ensure these assumptions are satisfied. Additionally if K is invertible, then it no longer appears in the results provided.

Finally, it is important to highlight a few contrasts between the cost function in \eqref{eqn: cost_tracking} and the typical assumptions in the online convex optimization literature.  First, note that the feasible action set $G = \mathbb{R}^n$ is unbounded.  Second, note that gradient of $c_t$ can be arbitrarily large when $y_t$ and $Kx_t$ are far apart.  Thus, both of these are relaxed compared to what is typically studied in the online convex optimization literature.  We show in Section \ref{SEC:AVG_CASE} that, we can have sublinear regret even in this relaxed setting.

\subsection{Performance Metrics}

The performance of online algorithms for SOCO problems is typically evaluated via two performance metrics: \emph{regret} and the \emph{competitive ratio}.  Regret is the dominant choice in the machine learning community and competitive ratio is the dominant choice in the online algorithms community.  The key difference between these measures is whether they compare the performance of the online algorithm to the offline optimal static solution or the offline optimal dynamic solution. Specifically, the optimal offline \emph{static} solution, is\footnote{One switching cost is incurred due to the fact that we enforce $x_0 = 0$.}
\begin{equation}
STA = \underset{{x \in G}}{\text{argmin}} \sum_{t=1}^{T} c_t(x) + \beta ||x||,
\label{eqn: static-def}
\end{equation}
and the optimal \emph{dynamic} solution is
\begin{align}
OPT &= \underset{(x_1, \ldots, x_T) \in G^T}{\text{argmin}}\sum_{t=1}^{T}c_t(x_t) + \beta ||(x_t - x_{t-1})||.
\end{align}

\begin{definition}\label{def: cr}
The \textbf{regret} of an online algorithm, $ALG$, is less than $\rho(T)$ if the following holds:
\begin{align}
\underset{(c_1, \dots, c_T) \in \mathcal{C}^T}\sup \cost(ALG) - \cost(STA) \le \rho(T).
\end{align}
\label{def: regret}
\end{definition}
\begin{definition}
An online algorithm $ALG$ is said to be $\rho(T)$-\textbf{competitive} if the following holds:
\begin{align}
\underset{(c^1, \dots, c^T) \in \mathcal{C}^T}\sup \frac{\cost(ALG)}{\cost(OPT)} \le \rho(T)
\end{align}
\end{definition}

The goals are typically to find algorithms with a (small) constant competitive
ratio (``constant-competitive'') and to find online algorithms with {sublinear
regret}, i.e., an algorithm $ALG$ that has regret $\rho(T)$
bounded above by some $\hat\rho(T)\in o(T)$; note that $\rho(T)$ may be
negative if the concept we seek to learn varies dynamically. Sublinear regret
is also called ``no-regret'', since the time-average loss of the online algorithm goes to zero as $T$ grows.
%\[ \lim_{T \rightarrow \infty}  \frac{1}{T}\left[ \sup_{(c_1, \ldots, c_T) \in \mathcal{C}} \cost(ALG) - \cost(STA)\right] \rightarrow 0. \]
%For this reason, such an online algorithm is also called a \emph{no-regret} algorithm.

\subsection{Background}

To this point, there are large literatures studying both the designs of no-regret algorithms and the design of constant-competitive algorithms. However, in general, these results tell a pessimistic story.

In particular, on a positive note, it is possible to design simple, no-regret algorithms, e.g., \emph{online gradient descent} (OGD) based algorithms \cite{zinkevich2003, hazan2007} and \emph{Online Newton Step and Follow the Approximate Leader} algorithms \cite{hazan2007}.  (Note that the classical setting does not consider switching costs; however, \cite{andrew2013} shows that similar regret bounds can be obtained when switching costs are considered.)

However, when one considers the competitive ratio, results are much less optimistic.  Historically, results about competitive ratio have considered weaker assumptions, i.e., the cost functions $c_t$ and the action set $G$ can be nonconvex, and the switching cost is an arbitrary metric $d(x_t, x_{t-1})$ rather than a seminorm $||x_t - x_{t-1}||$. The weakened assumptions, together with the tougher offline target for comparison, leads to the fact that most results are ``negative''. For example, \cite{borodin1992} has shown that any deterministic algorithm must be $\Omega(n)$-competitive given metric decision space of size $n$. Furthermore, \cite{blum1992} has shown that any randomized algorithm must be $\Omega(\sqrt{\log n / \log\log n})$-competitive. To this point, positive results are only known in very special cases. For example, \cite{lin2013} shows that, when $G$ is a one dimensional normed space, there exists a deterministic online algorithm that is 3-competitive.

Results become even more pessimistic when one asks for algorithms that perform well for both competitive ratio and regret.  Note that performing well for both measures is particularly desirable for many networking and cloud computing applications where it is necessary to both argue that a dynamic control algorithm provides benefits over a static control algorithm (sublinear regret) and is near optimal (competitive ratio).  However, a recent result in \cite{andrew2013} highlights that such as goal is impossible: even when the setting is restricted to a one dimensional normed space with linear cost functions no online algorithm can simultaneously achieve sublinear regret and constant competitive ratio \footnote{\scriptsize Note that this impossibility is not the result of the regret being additive and the competitive ratio being multiplicative, as \cite{andrew2013} proves the parallel result for competitive difference, which is an additive comparison to the dynamic optimal.}.

\section{Modeling Prediction Error}
\label{SEC:MODEL}  

The adversarial model underlying most prior work on online convex optimization has led to results that tend to be pessimistic; however, in reality, \emph{algorithms can often use predictions about future cost functions in order to perform well}.

Knowing information about future cost functions is clearly valuable for smoothed online convex optimization problems, since it allows you to better justify whether it is worth it to incur a switching cost during the current stage. Thus, it is not surprising that predictions have proven valuable in practice for such problems.

Given the value of predictions in practice, it is not surprising that there have been numerous attempts to incorporate models of prediction error into the analysis of online algorithms. We briefly expand upon the worst-case and stochastic approaches described in the introduction to motivate our approach, which is an integration of the two.

\textbf{Worst-case models.} Worst-case models of prediction error tend to
assume that there exists a lookahead window, $w$, such that within that window,
a perfect (or near-perfect, e.g., error bounded by $\varepsilon$) prediction is
available.  Then, outside of that window the workload is adversarial.  A
specific example is that, for any $t$ the online algorithm knows $y_t, \ldots,
y_{t+w}$ precisely, while $y_{t+w+1},\ldots$ are adversarial.

Clearly, such models are both \emph{too optimistic} about the the predictions
used and \emph{too pessimistic} about what is outside the prediction window.
The result is that algorithms designed using such models tend to be too
trusting of short term predictions and too wary of unknown fluctuations outside
of the prediction window.  Further, such models tend to underestimate the value
of predictions for algorithm performance. To illustrate this, we establish
the following theorem.

\begin{theorem} \label{thm: worstcaselower}
For any constant $\gamma > 0$ and any online algorithm $A$ (deterministic or
randomized) with constant lookahead $w$,
%there exists a family of SOCO instances, indexed by the number $T$ of cost functions, such that
either the competitive ratio of the algorithm is at least  $\gamma$ or its regret, is $\Omega(T)$. Here $T$ is the number of cost functions in an instance.
\end{theorem}
%\adam{Be sure to verify this theorem statement}
%\lachlan{I removed reference to ``SOCO instances'' and ``for large enough $T$'', since these are not meaningful for the standard definitions of CR and Regret, which are worst-cases over all instances, but they are used in the theorem quoted by the proof of this theorem; is that a problem?}

%\begin{proof}[sketch]
%The proof reduces to this from the
%problem of optimization with one-step prediction, for which was shown
%in~\cite{andrew2013} that no algorithm achieves both constant CR and sublinear
%regret.  The reduction expands the length of the sequence by a factor of
%$w+1$ by padding the input sequence with $w$ zero-cost
%functions between each pair of original cost functions.
%\end{proof}

The above theorem focuses on the ``easiest'' worst-case model, i.e., where the
algorithm is allowed perfect lookahead for $w$ steps.  Even in this case, an online algorithm must have super-constant lookahead in order to simultaneously have sublinear regret and a constant competitive ratio.  Further, the proof (given in Appendix~\ref{app:worstcase}) highlights that this holds even in the scalar setting with linear cost functions.  Thus, worst-case models are \emph{overly pessimistic} about the value of prediction.

\textbf{Stochastic models.} Stochastic models tend to come in two forms: (i) i.i.d.\ models or (ii) detailed models of stochastic processes and specific predictors for those processes.

In the first case, for reasons of tractability, prediction errors are simply assumed to be i.i.d.\ mean zero random variables.  While such an assumption is clearly analytically appealing, it is also quite simplistic and ignores many important features of prediction errors. For example, in reality, predictions have increasing error the further in time we look ahead due to correlation of predictions errors in nearby points in time. Further, predictions tend to be updated or refined as time passes. These fundamental characteristics of  predictions cannot be captured by the i.i.d.\ model.

In the second case, which is common in control theory, a specific stochastic model for the underlying process is assumed and then an optimal predictor (filter) is derived.  Examples here include the derivation of Weiner filters and Kalaman filters for the prediction of wide-sense stationary processes and linear dynamical systems respectively, see \cite{kailath2000}.  While such approaches avoid the pessimism of the worst-case viewpoint, they instead tend to be fragile to the underlying modeling assumptions.  In particular, an online algorithm designed to use a particular filter based on a particular stochastic model lacks the robustness to be used in settings where the underlying assumptions are not valid.
%, the derivation of weighted moving average predictors for random walks with drift \addcites,
% derivation of Kalman filters for the prediction of  
\subsection{A General Prediction Model}\label{sec:general_model}

A key contribution of this paper is the development of a model for studying predictions that provides a middle ground between the worst-case and the stochastic viewpoints.  The model we propose below seeks a middle ground by not making any assumption on the underlying stochastic process or the design of the predictor, but instead making assumptions only on the form of the error of the predictor.  Thus, it is agnostic to the predictor and can be used in worst-case analysis with respect to the realization of the underlying cost functions.

Further, the model captures three important features of real predictors:  (i) it allows for correlations in prediction errors (both short range and long range); (ii) the quality of predictions decreases the further in the future we try to look ahead; and (iii) predictions about the future are refined as time passes.

Concretely, throughout this paper we model prediction error via the following equation:
\begin{equation}
 y_t = y_{t|\tau} + \sum_{s=\tau+1}^t f(t - s) e(s).
 \label{eqn: prediction_model}
\end{equation}
Here, $y_{t|\tau}$ is the prediction of $y_t$ made at time $\tau<t$. Thus, $y_t-y_{t|\tau}$ is the prediction error, and is specified by the summation in \eqref{eqn: prediction_model}.  In particular, the prediction error is modeled as a weighted linear combination of per-step noise terms, $e(s)$ with weights $f(t-s)$ for some deterministic impulse function $f$. The key assumptions of the model are that $e(s)$ are i.i.d.\ with mean zero and positive definite covariance $R_e$; and that $f$ satisfies $f(0)=I$ and $f(t)=0$ for $t<0$.  Note that, as the examples below illustrate, it is common for the impulse function to decay as $f(s) \sim 1/s^\alpha$. As we will see later, this simple model is flexible enough to capture the prediction error that arise from classical filters on time series, and it can represent all forms of stationary prediction error by using appropriate forms of $f$.

Some intuition for the form of the model can be obtained by expanding the summation in \eqref{eqn: prediction_model}. In particular, note that for $\tau=t-1$ we have
{\begin{equation}
 y_t - y_{t|t-1} = f(0) e(t) = e(t),
 \label{eqn: prediction_model_onestep}
\end{equation}}
which highlights why we refer to $e(t)$ as the per-step noise.

Further, expanding the summation further gives
{\small\begin{align}
 y_t - y_{t|\tau} &= f(0)e(t) + f(1)e(t-1) + \ldots + f(t -\tau-1) e(\tau+1).
 \label{eqn: prediction_model_expanded}
\end{align}}
Note that the first term is the one-step prediction error $y_t - y_{t|t-1}$; the first two terms make up the two-step prediction error $y_t - y_{t|t-2}$; and so on.  This highlights that predictions in the model have increasing noise as one looks further ahead in time and that predictions are refined as time goes forward.

Additionally, note that the form of \eqref{eqn: prediction_model_expanded} highlights that the impulse function $f$ captures the degree of short-term/long-term correlation in prediction errors. Specifically, the form of $f(t)$ determines how important the error $t$ steps in the past is for the prediction.  Since we assume no structural form for $f$, complex correlation structures are possible.

Naturally, the form of the correlation structure plays a crucial role in the
performance results we prove.  But, the detailed structure is not important,
only its effect on the aggregate variance.  Specifically, the impact of the
correlation structure on performance is captured through the following two
definitions, which play a prominent role in our analysis. First, for any $w >
0$, let $||f_w||^2$ be the two norm of prediction error covariance over $(w+1)$ steps of prediction, i.e.,
{\begin{equation}
||f_w||^2 =  \mathrm{tr} (\ex [\delta y_w  \delta y_w^T]) = \mathrm{tr}(R_e
\sum_{s=0}^{w} f(s)^T  f(s)),
\label{eqn: def_f_omega}
\end{equation}}
where $\delta y_w^T = y_{t+w} - y_{t+w|t-1} = \sum_{s=t}^{t + w}f(t+w-s)e(s)$. The derivation of \eqref{eqn: def_f_omega} is found in the proof of Theorem \ref{thm: cd-afhc}.

Second, let $F(w)$ be the two norm square of the projected cumulative prediction error covariance, i.e., 
{\small\begin{equation}
 F(w) = \sum_{t=0}^{w} \ex ||KK^\dagger \delta y_{w}||^2 = \mathrm{tr}(R_e
 \sum_{s=0}^{w} (w - s + 1) f(s)^T KK^{\dagger} f(s)). \label{eqn: def F}
\end{equation}}
Note that $KK^{\dagger}$ is the orthogonal projector onto the range space of $K$. Hence it is natural that the definitions are over the induced norm of $KK^{\dagger}$ since any action chosen from the space $F$ can only be mapped to the range space of $K$ i.e. no algorithm, online or offline, can track the portion of $y$ that falls in the null space of $K$.

Finally, unraveling the summation all the way to time zero highlights that the process $y_t$ can be viewed as a random deviation around the predictions made at time zero, $y_{t|0}:=\hat{y}_t$, which are specified externally to the model:
{\begin{equation}
 y_t = \hat{y}_t + \sum_{s=1}^t f(t - s) e(s).
 \label{eqn: prediction_model_timezero}
\end{equation}}
This highlights that an instance of the online convex optimization problem can be specified via either the process $y_t$ or via the initial predictions $\hat{y}_t$, and then the random noise from the model determines the other.  We discuss this more when defining the notions of regret and competitive ratio we study in this paper in Section \ref{sec_performance_metric_predictions}.

\subsection{Examples}
While the form of the prediction error in the model may seem mysterious, it is
quite general, and includes many of the traditional models as special cases.
For example, to recover the worst-case prediction model one can set, $\forall
t, e(t) = 0$ and $\hat{y}_{t'}$ as unknown $\forall t' > t + w$ and then take the worst case over $\hat{y}$.  Similarly, a common approach in robust control is to set $f(t) = \begin{cases} I , & t=0; \\ 0, & t \ne 0\end{cases}$, $ |e(s)| < D, \forall s$ and then consider the worst case over $e$.

Additionally, strong motivation for it can be obtained by studying the predictors for common stochastic processes.  In particular, the form of \eqref{eqn: prediction_model} matches the prediction error of standard Weiner filters \cite{wiener1949} and Kalman filters \cite{kalman1960}, etc. To highlight this, we include a few brief examples below.
\begin{example}[Wiener Filter] \label{ex: wiener}
Let $\{y_t\}_{t=0}^T$ be a wide-sense stationary stochastic process with $\ex[y_t] = \hat{y}_t$, and covariance $\ex [(y_i - \hat{y}_i)(y_j - \hat{y}_j)^T] = R_y(i - j), $ i.e., the covariance matrix $R_y >0 $ of $y = [y_1\  y_2 \ \ldots y_T]^T$ is a Toeplitz matrix.
%Using a causal Wiener filter to predict $y_t$ at time $\tau$ leads to the linear estimator ${y}_{t|\tau} = K[y_1, \ldots, y_\tau]^T$ such that  $y_{t|\tau} = \arg\min \ex_{\tau} (y_t - K'[y_1, \ldots, y_\tau]^T)^2$. %One way to find $y_{t|\tau}$ efficiently is by finding $e$, the \textit{innovation process} of $y$, which is defined recursively:
%\[e(0) = y_0; \quad e(t+1) = y_{t+1} - \sum_{s=0}^t \langle y_{t+1}, e(s)\rangle ||e(s)||^{-2}e(s). \]
%where we define the inner product and norm of random variables by $\langle x, y\rangle = \ex[xy]$, and $||x||^2 = \ex[x^2]$\footnote{This construction ensures that $e(t)$ is uncorrelated to $y_1, \ldots, y_{t-1}$, which intuitively represents the ``new information'' that arrives at time $t$}. Assuming the stationary process has been run long enough, $\mathrm{Var}(e(t))$ tends to a constant $R_e$. \adam{Can we shrink the discussion above and include a citation instead? ...like in the case of the Kalman filter}  Then, it is straightforward to show that the best causal linear estimator is given by
The corresponding $e(s)$ in the Wiener filter for the process is called the ``innovation process'' and can be computed via the Wiener-Hopf technique \cite{kailath2000}. Using the innovation process $e(s)$, the optimal causal linear prediction is
{\small\[
y_{t|\tau} = \hat{y}_t + \sum_{s=1}^\tau \langle y_t , e(s) \rangle ||e(s)||^{-2} e(s),
\]}
and so the correlation function $f(s)$ as defined in \eqref{eqn: prediction_model} is
{\begin{equation}
f(s) = \langle y_s , e(0) \rangle ||e(0)||^{-2} = R_y(s) R_e^{-1},
\end{equation}}
which yields
{\small\begin{align*}
||f_w||^2 = \frac{1}{R_e} \sum_{s=0}^{w} R_y(s)^2 \text{ and }
F(w) = \frac{1}{R_e} \sum_{s=0}^{w} (w -s + 1) R_y(s)^2.
\end{align*}}
\end{example}

\begin{example} [Kalman Filter] \label{ex:kalman} Consider a stationary dynamical system described by the hidden state space model
{\small\begin{align*}
x'_{t+1} = A x'_t + B u_t, \quad y_t = C x'_t + v_t.
\end{align*}}
where the $\{u_t, v_t, x_0\}$ are $m \times 1, p \times 1$, and $n \times 1$-dimensional random variables such that
{\[ \left\langle \begin{bmatrix}
u_i \\ v_j \\ x_0\end{bmatrix} , \begin{bmatrix} u_i \\ v_j \\ x_0 \\ 1 \end{bmatrix}\right\rangle = \begin{bmatrix}
Q \delta_{ij} & S\delta_{ij} & 0 & 0 \\
S^* \delta_{ij} & R\delta_{ij} & 0 & 0\\
0 & 0 & \Pi_0 & 0
\end{bmatrix}. \]}

The Kalman filter for this process yields the optimal causal linear estimator
${y}_{t|\tau} = K[y_1^T, \ldots, y_\tau^T]^T$ such that  $y_{t|\tau} = \arg\min
\ex_{\tau} ||y_t - K'[y_1^T, \ldots, y_\tau^T]^T||^2$.  When $t$ is large and
the system reaches steady state, the optimal prediction is given in the following recursive form \cite{kailath2000}:
{\small\begin{align*}
x'_{t+1 | t} &= A x'_{t | t-1} + K_{p} e(t),\  y_{0|-1} = 0, \ e(0) = y_0, \\
e(t) &= y_t - C x'_{t|t-1},
\end{align*}}
where $K_{p} = (A PC^* + B S)R_{e}^{-1}$,  is the Kalman gain, and $R_{e} = R +
C P C^*$ is the covariance of the innovation process $e_t$, and $P$ solves
{\small\[
P = A P A^* + C Q C^* - K_{p} R_{e}K_{p}^*.
\]}
This yields the predictions
{\begin{align*}
y_{t|\tau} &= \sum_{s=1}^\tau \langle y_t, e(s)\rangle R_e^{-1} e(s) \\
&= \sum_{s=1}^\tau CA^{t - s-1}(APC^* + BS)R_e^{-1} e(s).
\end{align*}}
Thus, for stationary Kalman filter, the prediction error correlation function
is
{\small\begin{equation}
f(s) = CA^{s-1}(APC^* + BS) R^{-1}_{e} = C A^{s-1} K_p,
\label{eqn: correlation_kalman}
\end{equation}}
which yields
{\small\begin{align*}
||f_w||^2 %&= \sum_{s=0}^{w} \mathrm{tr}( R_e f(s)^TKK^\dagger f(s)) \\
&= \sum_{s=0}^{w} \mathrm{tr}(R_e (CA^{s-1} K_p)^TKK^\dagger (CA^{s-1} K_p)) \text{ and} \\
F(w) &=  \sum_{s=0}^{w} (w-s  +1)  \mathrm{tr}(R_e (CA^{s-1} K_p)^TKK^\dagger (CA^{s-1} K_p)).
\end{align*}}
\end{example}

\subsection{Performance Metrics}
\label{sec_performance_metric_predictions}

A key feature of the prediction model described above is that it provides a general stochastic model for prediction errors while not imposing any particular underlying stochastic process or predictor.  Thus, it generalizes a variety of stochastic models while allowing worst-case analysis.

More specifically, when studying online algorithms using the prediction model above, one could either specify the instance via $y_t$ and then use the form of \eqref{eqn: prediction_model} to give random predictions about the instance to the algorithm or, one could specify the instance using $\hat{y}:=y_{t|0}$ and then let the $y_t$ be randomly revealed using the form of \eqref{eqn: prediction_model_timezero}.  Note that, of the two interpretations, the second is preferable for analysis, and thus we state our theorems using it.

In particular, our setup can be interpreted as allowing an adversary to specify the instance via the initial (time $0$) predictions $\hat{y}$, and then using the prediction error model to determine the instance $y_t$.  We then take the worst-case over $\hat{y}$.  This corresponds to having an adversary with a ``shaky hand'' or, alternatively, letting the adversary specify the instance but forcing them to also provide unbiased initial predictions.

In this context, we study the following notions of (expected) regret and (expected) competitive ratio, where the expectation is over the realization of the prediction noise $e$ and the measures consider the worst-case specification of the instance $\hat{y}$.

\begin{definition}\label{def: new_regret}
We say an online algorithm $ALG$, has (expected) \textbf{regret} at most $\rho(T)$ if
\begin{align}
\sup_{\hat{y}} \ex_{e} [\cost(ALG) - \cost(STA)] \le \rho(T).
\end{align}
\end{definition}
\begin{definition}\label{def: new_cr}
We say an online algorithm $ALG$ is \textbf{$\rho(T)$-competitive} (in expectation) if
\begin{align}
\sup_{\hat{y}} \frac{\ex_{e}[\cost(ALG)]}{\ex_{e}[\cost(OPT)]} \le \rho(T).
\end{align}
\end{definition}

Our proofs bound the competitive ratio through an analysis of the competitive difference, which is defined as follows.

\begin{definition}\label{def: new_cd}
We say an online algorithm $ALG$ has (expected) \textbf{competitive difference} at most $\rho(T)$ if
\begin{align}
\sup_{\hat{y}} \ex_{e}\left[\cost(ALG) - \cost(OPT)\right] \le \rho(T).
\end{align}
\end{definition}

Note that these expectations are with respect to the prediction noise, $(e(t))_{t=1}^T$, and so $\cost(OPT)$ is also random. Note also that when $\cost(OPT) \in \Omega(\rho(T))$ and ALG has competitive difference at most $\rho(T)$, then the algorithm has a constant (bounded) competitive ratio.

\section{Averaging Fixed Horizon Control}
\label{sec:Afhc}

A wide variety of algorithms have been proposed for online convex optimization problems.  Given the focus of this paper on predictions, the most natural choice of an algorithm to consider is Receding Horizon Control (RHC), a.k.a., Model Predictive Control (MPC).

There is a large literature in control theory that studies RHC/MPC algorithms, e.g., 
\cite{garcia1989model, mayne2000constrained} and the references therein; and thus RHC is a popular choice for online optimization problems when
predictions are available, e.g., \cite{bemporad1999robust,wang2008cluster,kusic2009power,camacho2013model}.  However, recent results have highlighted
that while RHC can perform well for one-dimensional smoothed online
optimization problems, it does not perform well (in the worst case) outside of
the one-dimension case. Specifically, the competitive ratio of RHC with perfect
lookahead $w$ is $1+O(1/w)$  in the one-dimensional setting, but is
$1+\Omega(1)$ outside of this setting, i.e., the competitive ratio does not
decrease to $1$ as the prediction window $w$ increases~\cite{lin2012}.

In contrast, a promising new algorithm, Averaging Fixed Horizon Control (AFHC)
proposed by \cite{lin2012} in the context of geographical load balancing maintains good performance in high-dimensional settings, i.e., maintains a competitive ratio of $1+O(1/w)$\footnote{\scriptsize Note that this result assumes that the action set is bounded, i.e., for all feasible action $x, y$, there exists $D>0$, such that $|| x - y || < D$, and that there exists $ e_0 > 0, \text{ s.t. } c_t(0) \ge e_0, \forall t$.  The results we prove in this paper make neither of these assumptions.}. Thus, in this paper, we focus on AFHC.  Our results highlight that AFHC extracts the asymptotically optimal value from predictions, and so validates this choice.

As the name implies, AFHC averages the choices made by Fixed Horizon Control
(FHC) algorithms.  In particular, AFHC with prediction window size $(w+1)$ averages
the actions of $(w+1)$ FHC algorithms.

\begin{algorithm}[Fixed Horizon Control]
Let $\Omega_k = \left\{i : i \equiv k \mod (w + 1) \right\} \cap [-w, T]$ for
$k = 0, \dots, w$. Then $FHC^{(k)}(w+1)$, the $k$th FHC algorithm is
defined in the following manner. At timeslot $\tau \in \Omega_k$ (i.e., before
$c_\tau$ is revealed), choose actions $x^{(k)}_{FHC,t}$ for $t=\tau,\dots\tau+w$ as follows:

If $t\le 0$, $x^{(k)}_{FHC, t} = 0$. Otherwise, let $x_{\tau-1} = x^{(k)}_{FHC, \tau-1}$, and let
$(x^{(k)}_{FHC, t})_{t=\tau}^{\tau+w}$ be the vector that solves
\begin{align*}
\min_{x_\tau,\dots,x_{\tau+w}} \sum^{\tau + w}_{t=\tau} \hat{c}_t(x_t) +  \beta ||(x_t - x_{t-1})|| \;
\end{align*}
where $\hat{c}_t(\cdot)$ is the prediction of the future cost $c_t(\cdot)$ for $t = \tau, \ldots, \tau + w$. 
%\lachlan{I removed the factor of half from the switching cost, to reflect (2). Is that what was intended?}
\end{algorithm}
Note that in the classical OCO with $(w+1)$-lookahead setting, $\hat{c}_t(\cdot)$ is exactly equal to the true cost $c(\cdot)$.  Each FHC$^{(k)}(w+1)$ can be seen as a length $(w+1)$ fixed horizon control starting at position $k$. Given $(w+1)$ versions of FHC, AFHC is defined as the following: %For AFHC, it simply takes all the actions of FHC with different starting point, and take their average action. 
%\lachlan{What specifically was wrong with the definitions below?  Did you just
%want the discussion stripped, like I did?}

\begin{algorithm}[Averaging Fixed Horizon Control]

At timeslot $t \in {1, ...,T}$, $AFHC(w+1)$ sets
\begin{equation}
    x_{AFHC, t} = \frac{1}{w + 1}\sum_{k=0}^{w} x_{FHC, t}^{(k)}.
\end{equation}
\end{algorithm}

\section{Average-case Analysis}
\label{SEC:AVG_CASE}

We first consider the average-case performance of AFHC (in this section), and then consider distributional analysis (in Section \ref{SEC:CONC}).  We focus on the tracking problem in \eqref{eqn: cost_tracking} for concreteness and conciseness, though our proof techniques generalize.  Note that, unless otherwise specified, we use $||\cdot|| = ||\cdot||_2$.

Our main result shows that AFHC can simultaneously achieve sublinear regret and a constant competitive ratio using only a \emph{constant-sized} prediction window in nearly all cases that it is feasible for an online algorithm to do so.  This is in stark contrast with Theorem~\ref{thm: worstcaselower} for the worst-case prediction model.

\begin{theorem}\label{thm: general-simultaneous-afhc}
Let $w$ be a constant.
AFHC($w+1$) is constant-competitive whenever $\inf_{\hat{y}}\ex_{e}[OPT] = \Omega(T)$ and has sublinear regret whenever $\inf_{\hat{y}}\ex_e[STA]\geq \alpha_1 T - o(T)$, for
$ \alpha_1=4V + 8B^2 $,  where
\begin{align}
    V &= \frac{\beta||K^{\dagger}||_1||f_w|| +
    3\beta^2||(K^TK)^{-1}\mathbbm{1}|| + F(w)/2}
    	      {w+1} \label{eq: def-V} \\
    B &= \beta||(K^T)^{\dagger}\mathbbm{1}||. \label{eq: def-B} 
\end{align} 
\end{theorem}
and $||M||_1$ denotes the induced 1-norm of a matrix M
%Importantly, Theorem \ref{thm: general-simultaneous-afhc} only guarantees that AFHC achieves simultaneous guarantees when there is sufficient variability in the instance, i.e., \eqref{eqn: maintheoremcondition} and \eqref{eqn: maintheoremcondition_regret}.  These requirements may seem unexpected, but they are actually (nearly tight) natural restrictions for any online algorithm.

%Both conditions can be interpreted as saying that the ``signal'' in the instance must be large enough to be distinguished from the ``noise'' in the predictor.  In particular, the one-step prediction errors, $e(t) = y_t - y_{t|t-1}$, are i.i.d.\ and have non-zero variance, and so $$
%    \ex\left[\sum_{t=1}^{T} ||y_t - y_{t|t-1}||^2\right] \in \Omega(T).
%$$
%Since $y_{t|t-1}$ is the best estimate for $y_t$ available to an online algorithm, this error in prediction over the time horizon must be incurred by any online algorithm. Consequently, we need the ``signal'', the fluctuations in $y_t$, to be at least of the same order as the noise in the predictions, $\Omega(T)$, to track it.

%To make the above discussion more precise, note that the condition in \eqref{eqn: maintheoremcondition} is enough to ensure that $\text{cost}(OPT)\in\Omega(T)$ (see the proof of Theorem \ref{thm: general-simultaneous-afhc}).  This conclusion is necessary for \emph{any} online algorithm to have a hope of achieving a constant competitive ratio because, as we prove below, any online algorithm must have $\ex_e [\text{cost}(ALG)] = \Omega(T)$.

Theorem \ref{thm: general-simultaneous-afhc} imposes bounds on the expected costs of the dynamic and static optimal in order to guarantee a constant competitive ratio and sublinear regret.  These bounds come about as a result of the noise in predictions.  In particular, prediction noise makes it impossible for an online algorithm to achieve sublinear expected cost, and thus makes it infeasible for an online algorithm to compete with dynamic and static optimal solutions that perform too well.
This is made formal in Theorems \ref{thm: average_hardness_Omega_T} and \ref{thm: average_hardness},
which are proven in Appendix~\ref{sec:proof_ave_case}.  Recall that $R_e$ is
the covariance of an estimation error vector, $e(t)$.

\begin{theorem}\label{thm: average_hardness_Omega_T}\sloppy
Any online algorithm $ALG$ that chooses $x_t$ using only (i)~internal randomness independent of $e(\cdot)$ and (ii)~predictions made up until time $t$, has expected cost $\ex_e [\cost(ALG)] \geq \alpha_2 T + o(T)$, where
$ \alpha_2 = \frac{1}{2}||R_e^{1/2}||_{KK^\dagger}^2.$
\end{theorem}

\begin{theorem}\label{thm: average_hardness}
Consider an online algorithm $ALG$ such that  $\ex_e [\cost(ALG)] \in o(T)$. The actions, $x_t$, of $ALG$ can be used to produce one-step predictions $y'_{t|t-1}$, such that mean square of the one-step prediction error is smaller than that for $y_{t|t-1}$, i.e., $\ex_e || y_t - y'_{t|t-1}||^2 \le \ex_e ||y_t - y_{t|t-1} ||^2,$ for all but sublinearly many $t$.
\end{theorem}

Theorem~\ref{thm: average_hardness_Omega_T}
implies that it is impossible for any online algorithm that uses extra information (e.g., randomness) independent of the prediction noise to be constant competitive if $\ex_{e}[\cost(OPT)] = o(T)$ or to have sublinear regret if $\ex_{e}[\cost(STA)] \leq (\alpha_2-\varepsilon)T + o(T)$, for $\varepsilon>0$.

Further, Theorem~\ref{thm: average_hardness} states that if an online algorithm does somehow obtain asymptotically smaller cost than possible using only randomness independent of the prediction error,  then it must be using more information about future $y_t$ than is available from the predictions. This means that the algorithm can be used to build a better predictor. 
%This is unreasonable since, if this is true, the algorithm can be used to build a better predictor than was being used initially.

Thus, the consequence of  Theorems \ref{thm: average_hardness_Omega_T} and \ref{thm: average_hardness} is the observation that the condition in Theorem \ref{thm: general-simultaneous-afhc} for the competitive ratio is tight and the condition in Theorem \ref{thm: general-simultaneous-afhc} for regret is tight up to a constant factor, i.e., $\alpha_1$ versus $\alpha_2$.  (Attempting to prove matching bounds here is an interesting, but very challenging, open question.)

%\anish{Note it is not trivial for any online algorithm to be constant-competitive if cost(OPT) in $\Omega(T)$ as our action space is unbounded. Additionally, we specifically find AFHC to be V-competitive when cost(OPT) in $\Omega(T)$.} \adam{I don't think we need to state this -- and the V-competitive is not actually true I don't think.}

In the remainder of the section, we outline the analysis needed to obtain Theorem \ref{thm: general-simultaneous-afhc}, which is proven by combining Theorem \ref{thm: cd-afhc} bounding the competitive difference of AFHC and Theorem \ref{thm: regret-afhc} bounding the regret of AFHC. The analysis exposes the importance of the correlation in prediction errors for tasks such as determining the optimal prediction window size for AFHC. Specifically, the window size that minimizes the performance bounds we derive is determined not by the quality of predictions, but rather by how quickly error correlates, i.e., by $||f_\omega||^2$.

\subsection*{Proof of Theorem \ref{thm: general-simultaneous-afhc}}

The first step in our proof of Theorem \ref{thm: general-simultaneous-afhc} is to bound the competitive difference of AFHC.  This immediately yields a bound on the competitive ratio and, since it is additive, it can easily be adapted to bound regret as well.

The main result in our analysis of competitive difference is the following. This is the key both to bounding the competitive ratio and regret.

\begin{theorem}
\label{thm: cd-afhc}
The competitive difference of $AFHC(w+1)$ is $O(T)$ and bounded by:
\begin{equation}
\sup_{\hat{y}}\ex_e [\mathrm{cost}(AFHC) - \mathrm{cost}(OPT)]
\le V T
\label{eqn: ave_bound}
\end{equation}
where $V$ is given by~\eqref{eq: def-V}
%\begin{equation}
%V =\frac{1}{w + 1} \left(\beta||K^{\dagger}||_1||f_w|| +
%3\beta^2||(K^TK)^{-1}\mathbbm{1}|| + \half F(w) \right)
%\end{equation}
%where $M$ is given by~\eqref{eq: def-M}
\end{theorem}
Theorem \ref{thm: cd-afhc} implies that the competitive ratio of $AFHC$ is bounded by a constant when $\text{cost}(OPT)\in\Omega(T)$. 

The following corollary of Theorem \ref{thm: cd-afhc} is obtained by minimizing $V$ with respect to $w$.

\begin{corollary}\label{cor: optimal_lookahead}
For $AFHC$, the prediction window size that minimizes the bound in Theorem \ref{thm: cd-afhc} on competitive difference is a finite constant (independent of $T$) if $F(T)\in \omega(T)$ and is $T$ if there is i.i.d noise\footnote{Specifically $f(0) = I, f(t) = 0 \forall t > 0$}.
\end{corollary}

%\lachlan{It used to say ``$T$ if $F(T)\in O(T)$''.  However, it is indeterminate if $F(T)\in \Theta(T)$, and even if $F(T)\in o(T)$, it need not be $T$ for small $T$, before the asymptotics kick in.  Please check the new form.}

%This highlights that the optimal prediction window size is $\sim T$ if prediction errors are i.i.d., while it is a finite constant (independent of $T$) if predictions follow from a Kalman filter as in Example \ref{ex:kalman}.  

The intuition behind this result is that if the prediction model causes noise to correlate rapidly, then a prediction for a time step too far into the future will be so noisy that it would be best to ignore it when choosing an action under AFHC.  However, if the prediction model is nearly independent, then it is optimal for AFHC to look over the entire time horizon, $T$, since there is little risk from aggregating predictions.  Importantly, notice that the quality (variance) of the predictions is not determinant, only the correlation.

Theorem \ref{thm: cd-afhc} is proven using the following lemma (proven in the appendix) by taking expectation over noise.

\begin{lemma}\label{lem: cd-afhc-det}
The cost of $AFHC(w+1)$ for any realization of $y_t$ satisfies
\begin{align*}
 &\mathrm{cost}(AFHC) - \mathrm{cost}(OPT) \le \\
 &\frac{1}{w+1} \sum_{k=0}^{w} \sum_{\tau \in \Omega_k} \left(
 \beta||x^*_{\tau-1} - x^{(k)}_{\tau-1}||_1 + \sum_{t=\tau}^{\tau+w}\frac{1}{2}||y_t - y_{t|\tau-1}||_{KK^\dagger}^2\right).
\end{align*}
\end{lemma}

Next, we use the analysis of the competitive difference in order to characterize the regret of AFHC. In particular, to bound the regret we simply need a bound on the gap between the dynamic and static optimal solutions.

\begin{lemma}\label{lem: regret-open}
The suboptimality of the offline static optimal solution $STA$ can be bounded
below on each sample path by
\begin{align*}
&\mathrm{cost}(STA) - \mathrm{cost}(OPT)   \\
\ge  &\half\left(\sqrt{\sumt||y_t - \bar{y}||_{KK^\dagger}^2} - 2B\sqrt{T})\right)^2 - 2B^2 T - C
\end{align*}
where $\bar{y} = \frac{\sum^T_{t=1} y_t}{T}$, $B = \beta||(K^T)^\dagger\mathbbm{1}||_2 $ and $C =  \frac{\beta^2\mathbbm{1}^T(K^TK)^{-1}\mathbbm{1}}{2T}.$
\end{lemma}

Note that the bound above is in terms of $||(y_t - \bar{y})||_{KK^{\dagger}}^2$, which can be interpreted as a measure of the variability $y_t$. Specifically, it is the projection of the variation onto the range space of $K$.

Combining Theorem \ref{thm: cd-afhc} with Lemma \ref{lem: regret-open}
gives a bound on the regret of AFHC, proven in Appendix~\ref{sec:proof_ave_case}.

\begin{theorem}\label{thm: regret-afhc}
$AHFC$ has sublinear expected regret if
\[\inf_{\hat{y}}  \ex_e \sumt ||KK^\dagger(y_t - \bar{y})||^2 \ge (8V + 16B^2)T, \]
where $V$ and $B$ are defined in \eqref{eq: def-V} and \eqref{eq: def-B}.
%$V = \frac{1}{w+1}(\beta||K^{\dagger}||_1||f_w|| +
%3\beta^2||(K^TK)^{-1}\mathbbm{1}|| + \half F(w))$ and $B = \beta||(K^T)^{\dagger}\mathbbm{1}||_2$.
\end{theorem}
Finally, we make the observation that, for all instances of $y$:
{\small\begin{align*}
\cost(STA) &= \frac{1}{2} \sumt ||y_t - K_x||^2 + \beta||x||_1 \\
&\ge \frac{1}{2} \sumt ||(I-KK^\dagger) y_t + KK^\dagger y_t - Kx ||^2 \\
&= \frac{1}{2} \sumt ||(I-KK^\dagger) y_t||^2 + \half ||KK^\dagger y_t - Kx ||^2 \\
& \ge \half ||KK^\dagger(y_t - \bar{y})||^2.
\end{align*}}
Hence by Theorem \ref{thm: regret-afhc}, we have the condition of the Theorem.

%If $\sum^T_{t=1}\left(||(y_t - \bar{y})||_{KK^{\dagger}}^2\right)$ is
%$\omega(T)$, then for any window $w\in O(1)$, $AFHC(w)$ has negative regret for sufficiently large $T$.
%\adam{verify this theorem statement}
%\lachlan{The hypothesis is incredibly strong -- it says that the variability must increase unboundedly.  That makes this theorem very weak.  Is there anything we can say for the case that $\sum^T_{t=1}\left(||(y_t - \bar{y})||_{KK^{\dagger}}^2\right)$ is $\Theta(t)$?}
%\lachlan{Also, $y_t$ is random.  Does this hold for $\hat{y}_t$ satisfying the variability assumption?}
%\lachlan{I think the following theorem shows suitable regret for a system in which variance need not grow without bound.
%\begin{theorem}
%\label{thm: regret-afhc-lachlan}
%For any window $w\in O(1)$ such that, with high probability
%(i.e., almost surely with respect to $e$ as $T\rightarrow\infty$),
%\begin{align*}
%    &\frac{1}{T} \sum^T_{t=1}||y_t - \bar{y}||_{KK^{\dagger}}^2
%    \ge 16||(K^T)\mathbbm{1}||_2 \\
%    &+
%\frac{8}{w+1}(\beta||\mathbbm{1}^TK^\dagger||_\infty ||f_w|| +
%3\beta^2||(K^TK)^{-1}\mathbbm{1}||_1 + \frac{1}{2}F(w))
%\end{align*}
%it follows that
%$AFHC(w)$ has ``w.h.p. regret'' (whatever that means) $g(T)$ that is either %$o(T)$ or negative for sufficiently large $T$.
%\end{theorem}
%It would be nice to relate this to a simple condition on the expectation of the
%LHS of the hypothesis and the expected cost of AFHC, but the current proof takes squares and square-roots
%quite liberally.
%}

\section{Concentration Bounds}
\label{SEC:CONC}

The previous section shows that AFHC performs well in expectation, but it is also important to understand the distribution of the cost under AFHC. In this section, we show that, with a mild additional assumption on the prediction error $e(t)$, the event when there is a large deviation from the expected performance bound proven in Theorem \ref{thm: cd-afhc} decays exponentially fast.

The intuitive idea behind the result is the observation that the competitive difference of AFHC is a function of the uncorrelated prediction error $e(1), \ldots, e(T)$ that does not put too much emphasis on any one of the random variables $e(t)$. This type of function normally has sharp concentration around its mean because the effect of each $e(t)$ tends to cancel out.
%(Shown in Lemma \ref{lemma: g1-characterize} and Lemma \ref{lemma: g2-characterization})

For simplicity of presentation, we state and prove the concentration result for AFHC for the one dimensional tracking cost function
\[ \half \sumt (y_t - x_t)^2 + \beta |x_t - x_{t-1}|. \]
%\mathrm{Var}(e(t))
In this case, $R_e = \sigma^2$, and the correlation function $f: \mathbb{N} \rightarrow R$ is a scalar valued function. The results can all be generalized to the multidimensional setting.

Additionally, for simplicity of presentation, we assume (for this section only) that $\{e(t)\}_{t=1}^T$ are uniformly bounded, i.e., $\exists \epsilon>0$, s.t.\ $\forall t$, $|e(t)| < \epsilon$.  Note that, with additional effort, the boundedness assumption can be relaxed to the case of $e(t)$ being subgaussian, i.e., $\ex [\exp(e(t)^2 / \epsilon^2) ]\le 2$, for some $\epsilon >0$.\footnote{This involves more computation and worse constants in the concentration bounds. Interested readers are referred to Theorem 12 and the following remark of \cite{boucheron2004} for a way to generalize the concentration bound for the switching cost (Lemma \ref{lemma: g1-bound}), and Theorem 1.1 of \cite{rudelson2013} for a way to generalize the concentration bound for prediction error (Lemma \ref{lemma: g2-bound}).}

To state the theorem formally, let $VT$ be the upper bound of the expected competitive difference of $AFHC$ in
\eqref{eqn: ave_bound}.
%
%\[V = \frac{3\beta^2T}{w+1} + \frac{\beta T }{w+1}||f_w||_2 + \frac{T}{2(w+1)}
%F(w).\]
%
Given $\{\hat{y}_t \}_{t=1}^T$, the competitive difference of $AFHC$ is a
random variable that is a function of the prediction error $e(t)$.
The following theorem shows that the probability that the cost of $AFHC$
exceeds that of $OPT$ by much more than the expected value $VT$ decays rapidly.
\begin{theorem}
The probability that the competitive difference of $AFHC$ exceeds $VT$ is exponentially small, i.e., for any $u>0$:
{\small \begin{align*}
&\prob(\mathrm{cost}(AFHC) - \mathrm{cost}(OPT) > VT + u ) \\
\le   &\exp\left(\frac{-u^2}{8\epsilon^2 \frac{\beta^2T}{(w+1)\sigma^2}||f_w||^2}\right) +
\exp\left(\frac{-u^2}{16\epsilon^2 \lambda (2\frac{T}{w+1}F(w) + u)}\right) \\
 \le  &2\exp\left(\frac{-u^2}{a+ b u}\right),
\end{align*}}
where
$||f_w||^2 = (\sum_{t=0}^w |f(t)|^2) $, the parameter $\lambda$ of concentration
\[ \lambda \le  \sum_{t=0}^{w} (w-t) f(t)^2 =
\frac{1}{\sigma^2}F(w),\] and
%\begin{align*}
$ a  = 8{\epsilon^2} [T/(w+1)] \max(\frac{\beta^2}{\sigma^2} ||f_{w}||^2,
4\lambda F(w)), $
 $ b = 16 \epsilon^2 \lambda. $
%\end{align*}
\label{thm: afhc-concentration}
\end{theorem}

%Note that the first term in the concentration bound decays faster than the second term, we can get the following corollary:
%
%\begin{corollary}
%There exist $a, b,c >0$, such that for all sufficiently large $u$,
%\begin{align*}
%\prob(\mathrm{cost}(AFHC) - \mathrm{cost}(OPT)  > u + \ex VT ) \le  c\cdot \exp\left(\frac{-u^2}{a+ b u}\right),
%\end{align*}
%\label{cor: afhc-asymptotic}
%\end{corollary}

The theorem implies that the tail of the competitive difference of $AFHC$ has a Bernstein type bound. The bound decays much faster than the normal large deviation bounds obtained by bounding moments, i.e., Markov Inequality or Chebyshev Inequality. This is done by more detailed analysis of the  structure of the competitive difference of $AFHC$ as a function of $e = (e(1), \ldots, e(T))^T$.

Note that smaller values of $a$ and $b$ in Theorem \ref{thm:
afhc-concentration} imply a sharper tail bound. We can see that smaller
$||f_w||$ and smaller $F(w)$ implies the tail bound decays faster. Since higher
prediction error correlation implies higher $||f_w||$ and $F(w)$, Theorem \ref{thm: afhc-concentration} quantifies the intuitive idea that, the performance of AFHC concentrates more tightly around its mean when the prediction error is less correlated.
%Recall that $||f_w||^2$ is the prediction error variance for prediction $(w +
%1)$ steps ahead, and $F(w)$ is the aggregate prediction variance of predicting
%$1, \ldots, w+1$ steps ahead.

\subsection*{Proof of Theorem \ref{thm: afhc-concentration}}

To prove Theorem \ref{thm: afhc-concentration}, we start by decomposing the bound in Lemma \ref{lem: cd-afhc-det}.  In particular, Lemma \ref{lem: cd-afhc-det} gives
\begin{align}
\mathrm{cost}(AFHC) - \mathrm{cost}(OPT) &\le g_1 + g_2
%\frac{1}{w+1}
%\sum_{k=0}^w \sum_{\tau \in \Omega_k} \sum_{t=\tau}^{\tau + w } \beta |x^*_t - x^{(k)}_{t} |
%\\
% &+ \frac{1}{w+1}\sum_{k=0}^w \sum_{\tau \in \Omega_k} \sum_{t=\tau}^{\tau+w}\frac{1}{2}(y_t - y_{t|\tau-1})^2. \\
\label{eqn: separation}
\end{align}
where
\[g_1 = \frac{1}{w+1} \sum_{k=0}^w \sum_{\tau \in \Omega_k}  \beta |x^*_{\tau-1} - x^{(k)}_{\tau-1} | ,\]  represents loss due to the switching cost, and
\[g_2 = \frac{1}{w+1}\sum_{k=0}^w \sum_{\tau \in \Omega_k} \sum_{t=\tau}^{\tau
+ w}\frac{1}{2}(y_t - y_{t|\tau-1})^2,\]
 represents the loss due to the prediction error.

Let $V_1 = \frac{3\beta^2 T}{w + 1} + \frac{\beta T }{w + 1} ||f_w||_2$, and $V_2 =
\frac{T}{2(w + 1)}F(w)$. Note that $VT = V_1 + V_2$.  Then, by \eqref{eqn: separation},
 \begin{align}
& \prob(\cost(AFHC) - \cost(OPT) > u + VT ) \notag \\
%\le &\prob(g_1 + g_2 > V_1 + V_2 +  u) \notag \\
\le{} &\prob(g_1 > u/2 + V_1 \text{ or } g_2 >u/2 + V_2) \notag \\
\le{} &\prob(g_1>u/2 + V_1) + \prob(g_2>u/2 +V_2).
\label{eqn: afhc-cond-conc}
\end{align}

Thus, it suffices to prove concentration bounds for the loss due to switching cost, $g_1$, and the loss due to prediction error, $g_2$, deviating from $V_1$ and
$V_2$ respectively. This is done in the following. The idea is to  first prove
that $g_1$ and $g_2$ are functions of $e = (e(1), \ldots, e(T))^T$ that are not
``too sensitive'' to any of the elements of $e$, and then apply the method of
bounded difference \cite{mcdiarmid1989} and Log-Sobolev inequality
\cite{ledoux1999}.  Combining \eqref{eqn: afhc-cond-conc} with Lemmas
\ref{lemma: g1-bound} and \ref{lemma: g2-bound} below will complete the proof of Theorem \ref{thm: afhc-concentration}.

\paragraph*{Bounding the loss due to switching cost}
This section establishes the following bound on the loss due to switching:
\begin{lemma}\label{lemma: g1-bound}
The loss due to switching cost has a sub-Gaussian tail: for any $u>0$,
\begin{align}
\prob(g_1 > u + V_1)  \le  \exp\left(\frac{- u^2}{2\epsilon^2 \beta^2
\frac{T}{w+1} (||f_w||)^2}\right).
\label{eqn: g1_afhc}
\end{align}
\end{lemma}
To prove Lemma~\ref{lemma: g1-bound}, we introduce two lemmas.
Firstly, we use the first order optimality condition to bound $g_1$ above
by a linear function of $e = (e(1), \ldots, e(T))^T$ using the following lemma proved in the Appendix.
\begin{lemma}\label{lemma: g1-characterize}
The loss due to switching cost can be bounded above by
\begin{equation}\label{eq:upper_g1}
 g_1 \le \frac{3\beta^2 T }{w+1}
       + \frac{\beta}{w+1}
	    \sum_{k=0}^w\sum_{\tau \in \Omega_k}
	    \left| \sum_{s=1\vee(\tau-w-2)}^{\tau-1} f(\tau-1-s)e(s)
	    \right|
\end{equation}
%where $B_k$ is  given by
%{\scriptsize \begin{align*}
%& \begin{pmatrix}
% f(k-2) & \ldots  &  f(0)     \\
%  &    &   &  f(w) & \ldots & f(0)  & & \\
%  & & &   & \vdots & & &   \\
%  &  &      & &                  &                            &    f(w) & \ldots & f(0)
% \end{pmatrix}
% \end{align*}
%}
%where $f(k-2) := f(k - 1 + w  )$ if $k-2<0$. The number of rows of $B_k$ is
%$r_k = \lfloor \frac{T-k+1}{w + 1}\rfloor + 1$, and the number of columns of
%$B_k$ is $c_k = (k-1) + (w+1)(r_k - 1).$
%\label{lem: g1-characterization}
%\lachlan{I think it would be much clearer to write
%$$
%  g_1 \le  \frac{3\beta^2T}{w + 1}
%         + \frac{1}{w+1}\sum_{t=1}^T \beta(f_w \star e)(t),
%$$
%where
%$f_w(s) = f(s)$ if $0\le s \le w$ and 0 otherwise, $\star$
%denotes convolution, and $e$ is the error process, with $e(t)=0$ for $t<1$ if
%you like.  That is much easier than messing with edge effects of $B_k$, and
%also shows that the second term is $\Theta(T)$. }
%\niangjun{I think if we do that, $f_{w}(s)$ need to be much more complicated...}
\end{lemma}

Let $g_1'(e)$ be the second term of $g_1$.  Note
that the only randomness in the upper bound~\eqref{eq:upper_g1} comes from $g_1'$.

\begin{lemma}\label{lemma: g1-ex-bound}
The expectation of $g_1'(e)$ is bounded above by
\[  \ex_e g_1'(e) \le \frac{\beta T}{w + 1} ||f_w||_2. \]
\end{lemma}

With Lemma \ref{lemma: g1-ex-bound}, we can reduce \eqref{eqn: g1_afhc} to
proving a concentration bound on $g_1'(e)$, since
\begin{equation}
\prob(g_1 > u + V_1) \le \prob(g_1'  - \ex g_1'(e) \le   u).
\label{eqn: g1_concentration}
\end{equation}

To prove concentration of $g_1'(e)$, which is a function of a collection of
independent random variables, we use the method of bounded difference,
i.e., we bound the difference of $g_1'(e)$ where one component of $e$ is
replaced by an identically-distributed copy. Specifically, we use the following
lemma, the one-sided version of one due to McDiarmid:
\begin{lemma}[\cite{mcdiarmid1989}, Lemma 1.2]\label{lemma: bounded difference}
Let $X = (X_1, \ldots, X_n)$ be independent random variables
and $Y$ be the random variable $f(X_1, \ldots, X_n)$, where function $f$
satisfies
\[
    |f(x) - f(x_k')| \le c_k
\]
whenever $x$ and $x_k'$ differ in the $k$th coordinate.
Then for any $t>0$,
\[\prob(Y - \ex Y > t) \le \exp\left(\frac{-2t^2}{\sum_{k=1}^n c_k^2}\right).\]
\end{lemma}
%\begin{lemma}[\cite{mcdiarmid1989}, Lemma 1.2]
%Let $X = (X_1, \ldots, X_n)$ be a collection of independent random variables, and $X_k' = (X_1, \ldots, X'_k, \ldots, X_n)$ be the collection of random vector that differ from $X$ only at the $k^{th}$ coordinate by replacing an independent and identical copy of $X_k$, if the function $f$ satisfies
%\[
%    |f(X) - f(X_k')| \le c_k
%\]
%then let $Y$ be the random variable $f(X_1, \ldots, X_n)$, for any $t>0$,
%\[\prob(Y - \ex Y > t) \le \exp\left(\frac{-2t^2}{\sum_{k=1}^n c_k^2}\right).\]
%\label{lemma: bounded difference}
%\end{lemma}

\begin{proof}[of Lemma \ref{lemma: g1-bound}]
%To use  Lemma \ref{lemma: bounded difference},
Let $e = (e(1), \ldots , e(T))^T$, and $e_k' = (e(1), \ldots, e'(k), \ldots, e(T))^T$ be formed by replacing
$e(k)$ with an independent and identically distributed copy $e'(k)$.  Then
\begin{align*}
| g_1(e) - g_1(e_k') |  & \le \frac{1}{w+1} \beta \sum_{m=0}^w |f(m) (e(k) - e'(k))| \\
& \le \frac{2}{w+1} \epsilon \beta \sum_{m=0}^w |f(m)| =: c_k.
\end{align*}
Hence
\begin{align*}
\sum_{k=1}^T c_k^2 &= \frac{4\epsilon^2 \beta^2T}{(w+1)^2}  \left(\sum_{m=0}^w
|f(m)|\right)^2 \le 4\epsilon^2 \beta^2 \frac{T}{(w+1)\sigma^2} ||f_w||^2. \\
\end{align*}

By Lemma \ref{lemma: bounded difference},
\begin{align*}
&\prob(g_1'(e) - \ex g_1'(e) > u) \le \exp\left(\frac{-u^2}{2\epsilon^2 \beta^2
\frac{T}{(w+1)\sigma^2} (||f_w||)^2}\right).
%\label{eqn: g1_afhc}
\end{align*}
Substituting this into \eqref{eqn: g1_concentration} and \eqref{eqn: g1_afhc} finishes the proof.
\end{proof}

\paragraph*{Bounding the loss due to prediction error}
In this section we prove the following concentration result for the loss due to correlated prediction error.

\begin{lemma}\label{lemma: g2-bound}
The loss due to prediction error has Berstein type tail: for any $u>0$,
\begin{align}
\prob(g_2> u + V_2) \le \exp\left(\frac{-u^2}{8\epsilon^2 \lambda
(\frac{T}{w+1}F(w) + u)}\right).
\label{eqn: g2_afhc}
\end{align}
\end{lemma}

%\begin{proof}[of lemma \ref{lemma: g2-bound}]
To prove Lemma \ref{lemma: g2-bound}, we characterize $g_2$ as a convex
function of $e$ in Lemma \ref{lemma: g2-characterization}. We then show that this is a \emph{self-bounding} function. Combining convexity and self-bounding property of $g_2$, Lemma \ref{lemma: self-bound} makes use of the convex Log-Sobolev inequality to prove concentration of $g_2$.
\begin{lemma}\label{lemma: g2-characterization}
The expectation of $g_2$ is $\ex g_2 = V_2$, and $g_2$ is a convex quadratic
form of $e$.  Specifically, there exists a matrix $A\in \mathbb{R}^{T \times T}$, such that $g_2 = \frac{1}{2}||A e||^2$.
Furthermore, the spectral radius of $\lambda$ of $A A^T$ satisfies $\lambda \le F(w)$.
\end{lemma}

Hence, \eqref{eqn: g2_afhc} is equivalent to a concentration result of $g_2$:
\[\prob(g_2 > V_2 + u ) = \prob(g_2 - \ex g_2 > u). \]

The method of bounded difference used in the previous section is not good
for a quadratic function of $e$ because the uniform bound of $|g_2(e) - g_2(e'_k)|$ is too large since
\[ |g_2(e) - g_2(e'_k)| = \half |(e-e'(k))^TA^TA(e+e'(k)) |,\]
where the $(e + e'(k))$ term has $T$ non-zero entries and a uniform upper bound of this will be in $\Omega(T)$. Instead, we use the fact that the quadratic form is
self-bounding.  Let $h(e) = g_2(e) - \ex g_2(e)$. Then
\begin{align*}
||\nabla h(e)||^2
&~=~ ||A^T Ae||^2 ~=~  (Ae)^T({A A^T})(Ae) \\
&~\leq~ \lambda (Ae)^T(Ae) ~=~ 2 \lambda [h(e) + \ex V_2].
\end{align*}

We now introduce the concentration bound for a self-bounding
function of a collection of random variables.  The proof uses the convex
Log-Sobolev inequality~\cite{ledoux1999}.
\begin{lemma}\label{lemma: self-bound}
Let $f:\mathbb{R}^n\rightarrow\mathbb{R}$ be convex and random variable $X$ be supported on $[-d/2,d/2]^n$. If $\ex [f(X)]=0$ and $f$ satisfies the self-bounding property
	\begin{equation}
	\label{self_bound} ||\nabla f||^2 \le af + b,
	\end{equation}
for $a,b>0$, then the tail of $f(X)$ can be  bounded as
\begin{equation}
\mathbb{P}\left\{f(X) >t\right\} \le \exp\left(\frac{-t^2}{d^2(2b + at)}\right).
\end{equation}
\end{lemma}

%According to the bounded prediction error assumption, one has $|e|\leq\epsilon$ component-wise. Then,
Now to complete the proof of Lemma \ref{lemma: g2-bound}, apply Lemma \ref{lemma: self-bound} to the random variable $Z = h(e)$ to obtain
	\[\mathbb{P}\{g_2 - \ex g_2 >u\} \le \exp\left(-\frac{u^2}{8\lambda_{\max} \epsilon^2(2 V_2 + u)}\right) \]
for $t>0$, i.e.,
	\begin{align*} &\mathbb{P}\{g_2>u + v_2\} \le \exp\left(-\frac{u^2}{8\lambda_{\max} \epsilon^2(2 V_2 + t)}\right)\\
	 = &\exp\left(\frac{-u^2}{8\epsilon^2 \lambda (\frac{T}{w+1}F(w) + u)}\right) .\end{align*}
%this finishes the proof.
%\end{proof}

%Finally, we can prove Theorem \ref{thm: afhc-concentration}
%\begin{proof}[of Theorem \ref{thm: afhc-concentration}]
%Combining \eqref{eqn: afhc-cond-conc} with Lemma \ref{lemma: g1-bound} and Lemma \ref{lemma: g2-bound} completes the proof.
%\end{proof}

\section{Concluding Remarks}
\label{sec:disc}
Making use of predictions about the future is a crucial, but under-explored, area of online algorithms.  In this paper, we have introduced a general colored noise model for studying predictions.  This model captures a range of important phenomena for prediction errors including, general correlation structures, prediction noise that increases with the prediction horizon, and refinement of predictions as time passes. Further it allows for worst-case analysis of online algorithms in the context of stochastic prediction errors.

To illustrate the insights that can be gained from incorporating a general model of prediction noise into online algorithms, we have focused on online optimization problems with switching costs, specifically, an online LASSO formulation.  Our results highlight that a simple online algorithm, AFHC, can simultaneously achieve a constant competitive ratio and a sublinear regret in expectation in nearly any situation where it is feasible for an online algorithm to do so.  Further, we show that the cost of AFHC is tightly concentrated around its mean.

We view this paper as a first step toward understanding the role of predictions in the design of online optimization algorithms and, more generally, the design of online algorithms.  In particular, while we have focused on a particular, promising algorithm, AFHC, it is quite interesting to ask if it is possible to design online algorithms that outperform AFHC.  We have proven that AFHC uses the asymptotically minimal amount of predictions to achieve constant competitive ratio and sublinear regret; however, the cost of other algorithms may be lower if they can use the predictions more efficiently.

In addition to studying the performance of algorithms other than AFHC, it would also be interesting to generalize the prediction model further, e.g., by considering non-stationary processes or heterogeneous $e(t)$.

% \section*{Acknowledgement}
% This work is partly supported by NSF through CNS-1319820 and EPAS-1307794.

\bibliographystyle{abbrv}
{
\bibliography{socop_reference}
}

\appendix
\section{Proofs for Section \ref{SEC:MODEL}}
\subsection{Proof of Theorem \ref{thm: worstcaselower}}
\label{app:worstcase}
%\lachlan{To fit in 12+2 pages, I think we should give only a proof sketch in $5\sim10$ lines.  The concept is
%very simple.  (I wrote one, but someone removed it.)}

For a contradiction, assume that there exists an algorithm $\A'$ that achieves constant competitive ratio \emph{and} sublinear regret with constant lookahead. We can use algorithm $\A'$ to obtain another online algorithm $\A$ that achieves constant competitive ratio \emph{and} sublinear regret \emph{without} lookahead. This contradicts Theorem 4 of~\cite{andrew2013}, and we get the claim.

Consider an instance $\{c_1, c_2, \ldots, c_T \}$ without lookahead. We simply
``pad'' the input with $\ell$  copies of the zero function ${\bf 0}$ if $\A'$
has a lookahead of $\ell$. That is the input to $\A'$ is: $c_1, {\bf 0},\ldots, {\bf 0}, c_2, {\bf 0}, \ldots, {\bf 0}, c_3, {\bf 0},\ldots$

%$ c_1, \underbrace{{\bf 0}, \ldots, {\bf 0}}_{\ell}, c_2, \underbrace{{\bf 0}, \ldots, {\bf 0}}_{\ell}, c_3, {\bf 0},\ldots$.

%Algorithm $\A'$ incurs a cost with respect to cost function $c_t$ at time $(t-1)(\ell+1) + 1$.
We simulate $\A'$ and set the $t$th action of $\A$ equal to the $((t-1)(\ell+1) + 1)$th action of $\A'$. Note that the optimal values of the padded instance are equal to the optimal values of the given instance. Also, by construction, $\cost(\A) \leq \cost(\A')$. Therefore, if $\A'$ achieves constant competitive ratio \emph{and} sublinear regret then so does $\A$, and the claim follows.

\section{Proofs for Section~\ref{SEC:AVG_CASE}}
\label{sec:proof_ave_case}
\subsection{Proof of Theorem \ref{thm: average_hardness_Omega_T}}
\begin{proof}
Let $(x_{ALG,t})_{t=1}^T$ be the solution produced by online algorithm $ALG$.
Then
{\small
\begin{align*}
  \mathrm{cost}(ALG) \ge& \half \sumt ||y_t - K x_{ALG, t}||^2 \\
  =&\half \sumt||(I-KK^\dagger)y_t||^2 +  ||KK^\dagger y_t - K x_{ALG, t}||^2,
  \end{align*}}
by the identity $(I-KK^\dagger)K = 0$. Let $\epsilon_t = x_{ALG, t} - K^\dagger
y_{t|t-1}$, i.e.,
$    \epsilon_t
    % =& x_{ALG, t} - K^\dagger y_{t|t-1} \\
    	= x_{ALG, t} - K^\dagger (y_{t|0} - \sum_{s=1}^{t-1}f(t-s)e(s)).
$
Since all predictions made up until $t$ can be expressed in terms of
$y_{\cdot|0}$ and $e(\tau)$ for $\tau < t$,
which are independent of $e(t)$, and all other information (internal randomness)
available to ALG is independent of $e(t)$ by assumption, $\epsilon_t$ is independent of $e(t)$.
It follows that
{\small
\begin{align*}
&\ex_e[ \mathrm{cost}(ALG)] \ge \ex_e[ ||KK^\dagger y_t - K(K^\dagger y_{t|t-1} + \epsilon_t)||^2 ]\\
 %   \ge &\ex_e  \half \sumt ||(I-KK^\dagger)y_t +  KK^\dagger (y_t -  y_{t|t-1}) - K\epsilon_t||^2 \\
 %   = &  \ex_e  \half \sumt ||(I-KK^\dagger)y_t||^2 +  ||KK^\dagger (y_t -  y_{t|t-1}) - K\epsilon_t||^2 \\
    \label{eqn: sum_expectation}
    = &\half \sumt \ex_{e\setminus e(t)} \ex_{e(t)|e\setminus e(t)}
    		||KK^\dagger e(t)^T - K\epsilon_t||^2  \numberthis \\
    = &\half \sumt \ex_{e\setminus e(t)} (||R_e^{1/2}||_{KK^\dagger}^2 + ||(\ex_{e(t)} \epsilon_t\epsilon_t^T)^{1/2}||_{K^TK}^2) \\
    \ge &\frac{T}{2}||R_e^{1/2}||_{KK^\dagger}^2,
\end{align*}}
where the first equality uses the identity $(I-KK^\dagger)K = 0$, and the
second uses the independence of  $\epsilon_t$ and $e(t)$.
\end{proof}

\subsection{Proof of Theorem \ref{thm: average_hardness}}
By Theorem \ref{thm: average_hardness_Omega_T}, if
%Firstly, for online any algorithm $ALG$, let the solution produced by this online algorithm be $x_{ALG, 1}, \ldots, x_{ALG, T}$. Since the total cost of the algorithm is the sum of tracking cost and switching cost, this is greater than the tracking cost alone:
%\begin{align*}
%& \mathrm{cost}(ALG) \ge \sumt \half ||y_t - Kx_{ALG, t}||^2  \\
%= &\sumt \half  ||(I-KK^\dagger) y_t||^2 + \half ||KK^\dagger y_t - Kx_{ALG, t}||^2.
%\end{align*}
%where the last equality is because $(I - KK^\dagger) K = 0$.
%Let $x_{ALG, t} = K^\dagger y_{t|t-1} + \epsilon_t$, since $||(I-KK^\dagger)y_t||^2$ is nonnegative, the expectation of the cost can be lower bounded by:
%\begin{align*}
%&\ex [\mathrm{cost}(ALG)]\\
% \ge &\sumt \half  \ex ||KK^\dagger y_t - K(K^\dagger y_{t|t-1} + \epsilon_t) ||^2 \\
%%= & \half \sumt \ex ||  KK^\dagger e(t) - K\epsilon_t ||^2 \\
%= &\sumt  \half ||R_e^{1/2}||^2_{KK^\dagger, F} + \half ||(\ex \epsilon_t \epsilon_t^T)^{1/2}||^2_{K^TK, F}  - \ex[e(t)^T K \epsilon_t]\\
%%= &\frac{T}{2} ||R_e^{1/2}||^2_ {KK^\dagger, F}  + \sumt \half ||(\ex \epsilon_t \epsilon_t^T)^{1/2}||^2_{K^TK, F} -\sumt\ex[e(t)^T K \epsilon_t].
%\end{align*}
%where the first equality makes use of the fact that the property of pseudo-inverse that $KK^\dagger K = K$.
%
%Note that $\frac{T}{2} ||R_e^{1/2}||^2_{KK^\dagger, F}$ is of order $T$, but
$\ex[\mathrm{cost}(ALG)] \in o(T)$, there must be some $t$ such that $\epsilon_t$ is not independent of $e(t)$. By expanding the square term in \eqref{eqn: sum_expectation} and noting it is nonnegative:
{\small\[ \ex[e(t)^T K\epsilon_t] \le \half ||R_e^{1/2}||^2_{KK^\dagger, F} + \half ||(\ex \epsilon_t \epsilon_t^T)^{1/2}||^2_{K^TK, F}. \]}
Each nonzero $\ex[e(t)^TK\epsilon_t]$ can at most make one term in \eqref{eqn: sum_expectation} zero, since there are $T$ terms in \eqref{eqn: sum_expectation} that are each lower bounded by $\half|| R_e^{1/2}||^2$, and by assumption $\ex[\cost(ALG)]$ is sublinear. There can be at most a sublinear number  of $t$ such that $\ex[e(t)^T K\epsilon_t] = 0$.

For every other $t$, we must have $\ex[e(t)^T K \epsilon_t] > 0$.
Let $l_t = \ex[e(t)^T K \epsilon_t]>0$, and $a_t = ||(\ex\epsilon_t \epsilon_t^T)^{1/2}||^2_{KK^T, F}>0.$

Hence, at time $t$, the algorithm can produce prediction $y'_{t|t-1} = y_{t|t-1} + \frac{1}{w_t}K\epsilon_t$, where the coefficient $w_t$ is chosen later. Then the one step prediction error variance:
{\small\begin{align*}
\ex ||y_t - y'_{t|t-1}||^2 & = \ex || e(t) - \frac{1}{w_t}K\epsilon_t||^2 \\
&= ||R_e^{1/2}||_F^2 - \frac{2}{w_t} l_t + \frac{1}{w_t^2} a_t.
\end{align*}}
Pick any $w_t > a_t/2l_t$.  Then $\ex||y_t - y'_{t|t-1}||^2 < ||R_e^{1/2}||_F^2 = \ex ||y_t - y_{t|t-1}||^2.$ Hence $ALG$ can produce better one-step prediction for all but sublinearly many $t$.

\subsection{Proof of Lemma \ref{lem: cd-afhc-det}}
%\lachlan{The flow of this section needs work: there is too much interleaving
%between statements of lemmas and their proofs.  How about the following flow:
%Lemma 18; proof of Lemma 7; Definition of OPEN; Lemma 19; proof of Lemma 19;
%Proof of Lemma 18.} niangjun: changes implemented
To prove Lemma \ref{lem: cd-afhc-det}, we use the following Lemma.
\begin{lemma}
The competitive difference of $FHC$ with fixed $(w+1)$-lookahead for any realization is given by
\small{\begin{align*}
 \mathrm{cost}(FHC^{(k)}) \le &\mathrm{cost}(OPT) + \sum_{\tau \in \Omega_k}  \beta ||( x^*_{\tau-1} - x^{(k)}_{\tau-1})||_1 \\
& + \half \sum_{\tau \in \Omega_k}\sum_{t=\tau}^{\tau + w} || K K^\dagger (y_t - y_{t | \tau-1})||^2.
\end{align*}}
\label{lem: cd-fhc-det}
where $x^*_{t}$ is the action chosen by the dynamic offline optimal.
\end{lemma}

\begin{proof}[of lemma \ref{lem: cd-afhc-det}]
Note that $\mathrm{cost}(FHC^{(k)})$ is convex. The result then follows with a straightforward application of Jensen's inequality to Lemma \ref{lem: cd-fhc-det}.  By the definition of $AFHC$, we have the following inequality:
{\small\begin{align*}
\mathrm{cost}(AFHC) &\le \frac{1}{w+1} \sum_{k=0}^{w} \mathrm{cost}(FHC^{(k)})
\end{align*}}
By substituting the expression for $\mathrm{cost}(FHC^{(k)})$ into the equation above and simplifying, we get the desired result.
%&\le \mathrm{cost}(OPT) + \frac{1}{w+1} (\sum_{k=0}^{w} \sum_{\tau \in \Omega_k} \beta||(x^*_{\tau-1} - x^{(k)}_{\tau - 1})||_1  \\
%& + \sum_{k=0}^w\sum_{\tau \in \Omega_k}\sum_{t=\tau}^{\tau + w}\half||( y_t - y_{t|\tau})||_{ K K^\dagger}^2)\\
%&= \mathrm{cost}(OPT) + \frac{1}{w+1}( \sum_{k=0}^{w} \sum_{\tau \in \Omega_k} + \beta||(x^*_{\tau-1} - x^{(k)}_{\tau - 1})||_1 \\
%& + \sum_{l=0}^{w}\sum_{t=1}^{T}\half||(y_t - y_{t|t - l})||_{ K K^\dagger}^2)
\end{proof}

Before we prove Lemma \ref{lem: cd-fhc-det}, we first introduce a new algorithm we term $OPEN$. This algorithm runs an open loop control over the entire time horizon, $T$. Specifically, it chooses actions $x_t$, for $t \in {1,...,T}$, that solves the following optimization problem:
{\small\begin{align*}
{\text{min}}\frac{1}{2}\sum^{T}_{t=1}(y_{t|0} - Kx_t)^2 + \beta||(x_t - x_{t-1})||_1
\end{align*}}

$FHC^{(k)}$ can be seen as starting at $x^{k}_{FHC, \tau-1}$, using prediction $y_{\cdot | \tau-1}$, and running $OPEN$ from $\tau$ to $\tau+w$. Then repeating with updated prediction $y_{\cdot|\tau+\omega}$. We first prove the following Lemma characterizing the performance of $OPEN$.
\begin{lemma}
Competitive difference of $OPEN$ over a time horizon, $T$, is given by
{\small\begin{equation*}
\quad \mathrm{cost}(OPEN) - \mathrm{cost}(OPT) \le \sumt \half ||\hat{y}_t - y_t||_{KK^\dagger}^2
\end{equation*}}
\label{lem: cd-open-det}
\end{lemma}
\begin{proof}
Recall that the specific OCO we are studying is
{\small\begin{equation}
	\min_x  \sumt \half||y_t - Kx_t||^2 + \beta||(x_t - x_{t-1})||_1
	\label{eqn: gen-vector-l2}
\end{equation}}
where $x_t \in \mathbb{R}^n$, $y_t \in \mathbb{R}^m$, $K \in \mathbb{R}^{m \times n}$ and the switching cost, $\beta \in \mathbb{R}_{+}$.

We first derive the dual of \eqref{eqn: gen-vector-l2} by linearizing the $l_1$ norm which leads to the following equivalent expression of the objective above:
{\small\begin{align*}
\min_{x, z} & \half \sumt ||y_t - Kx_t||^2 + \beta\mathbbm{1}^T z_t \\
\mathrm{s.t. }\quad & z_t \ge x_t - x_{t-1}, z_t  \ge x_{t-1} - x_t, \quad \forall t.
\end{align*}}
Hence the Lagrangian is
{\small\begin{align*}
L(x,z; \bar{\lambda}, \underline{\lambda}) %&= \half \sumt ||y_t - Kx_t||^2 + \beta\mathbbm{1}^T z + \bar{\lambda}^T(x_t - x_{t-1} - z_t) \\
%& + \underline{\lambda}^T(x_{t-1}-x_t-z_t) \\
&= \half\sumt ||y_t - Kx_t||^2 + \langle \bar{\lambda}_t - \underline{\lambda}_t, x_t - x_{t-1}\rangle \\
& + \langle \beta\mathbbm{1} - (\bar{\lambda}_t + \underline{\lambda}_t), z_t \rangle.
\end{align*}}
where we take $\lambda_{T+1} = 0$ and $x_0 = 0$.

Let $\lambda_t = \bar{\lambda}_t - \underline{\lambda}_t$ and  $w_t = \bar{\lambda}_t + \underline{\lambda}_t$. Dual feasibility requires $w_t \le \beta\mathbbm{1}, \forall t$, which implies $-\beta\mathbbm{1} \le \lambda_t \le \beta\mathbbm{1}, \forall t$. Dual feasibility also requires $\langle \beta\mathbbm{1}-w_t, z_t\rangle = 0, \forall t$.

Now by defining $s_t = \lambda_{t} - \lambda_{t+1}$ and equating the derivative with respect to $x_t$ to zero, the primal and dual optimal $x^*_t, s^*_t$ must satisfy $K^T K x^*_t = K^Ty_t - s^*_t$.
%\begin{equation}
%K^T K x^*_t = K^Ty_t - s^*_t
%\label{eqn: normal}
%\end{equation}

Note by premultiplying the equation above by ${x^*_t}^T$, we have $\langle x^*_t, s^*_t \rangle= \langle Kx^*_t, y_t\rangle - ||Kx^*_t||^2$. If instead we premultiply the same equation by $(K^T)^\dagger$, we have after some simplification that $Kx^*_t  = (KK^\dagger) y_t - (K^T)^\dagger s^*_t$. We can now simplify the expression for the optimal value of the objective by using  the above two equations:
%\begin{equation}
%\langle x^*_t, s^*_t \rangle= \langle Kx^*_t, y_t\rangle - ||Kx^*_t||^2.
%\label{eqn: pd_inner}
%\end{equation}

%\begin{equation}
%Kx^*_t  = (KK^\dagger) y_t - (K^T)^\dagger s^*_t  \label{eqn: kx}
%\end{equation}

%\eqref{eqn: kx} and \eqref{eqn: pd_inner}
\begin{align}
&\cost(OPT) = \sumt \half ||y_t - Kx^*_t||^2 + \langle x^*_t, s^*_t\rangle  \notag\\
%&=\sumt \half ||y_t||^2 -\langle Kx^*_t, y_t \rangle + \half ||Kx^*_t||^2 + \langle Kx^*_t, y_t\rangle - ||Kx^*_t||^2 \notag\\
= &\sumt \half ||y_t ||^2 - \half||KK^\dagger y_t - (K^T)^\dagger s^*_t ||^2
\label{eqn: optimal_val}
\end{align}

Observe that \eqref{eqn: optimal_val} implies $s^*_t$ minimizes the following expression $\sumt ||KK^\dagger y_t - (K^T)^\dagger s^*_t ||^2$ over the constraint set $\mathcal{S} = \{s_t | s_t = \lambda_t - \lambda_{t+1}, -\beta\mathbbm{1} \le \lambda_t \le \beta\mathbbm{1} \text{ for } 1\le t \le T, \lambda_{T+1} = 0\}$.
%Let $\hat{y_t}$ refer to the prediction of $y_t$ at $t=0$ and let $\hat{x} =
%\{\hat{x}_1, \ldots, \hat{x}_T\}$ refer to the action chosen by $OPEN$ with
%prediction $\hat{y} = \{\hat{y}_1, \ldots \hat{y}_{T}\}$. Let $p(a; b)$ refers to the total cost of the primal given action $a$ and outcome $b$. We can now define the optimality gap for the open loop algorithm
{\small\begin{align*}
& \mathrm{cost}(OPEN) - \mathrm{cost}(OPT) = p(\hat{x}; y) - p(x; y) \\
=&p(\hat{x}; \hat{y}) - p(x;y) + p(\hat{x}; y) - p(\hat{x}; \hat{y}_t) \\
=& \sumt \half ||\hat{y}_t ||^2 - \half||K\hat{x}_t||^2  - \half ||y_t ||^2 + \half||Kx^*_t||^2 \\
&\quad + \half||y_t - K\hat{x}_t||^2 - \half||\hat{y}_t - K\hat{x}_t||^2 \\
%&=\sumt \half ||KK^\dagger y_t - (K^T)^\dagger s^*_t||^2 + \half ||\hat{y}_t - y_t||^2  \\
%&- \half ||(I-KK^\dagger)(\hat{y}_t - y_t) - (KK^\dagger y_t - (K^T)^\dagger\hat{s}_t)||^2
\end{align*}}
Expanding the quadratic terms, using the property of the pseudo-inverse that $K^\dagger K K^\dagger = K^\dagger$, and using the fact that $Kx_t^* = KK^\dagger y_t - (K^T)^\dagger s_t^*$, we have
%\begin{align*}
%&\langle (I-KK^\dagger)(\hat{y_t} - y_t), KK^\dagger y_t - (K^T)^\dagger\hat{s}_t \rangle = 0
%\end{align*}
{\small\begin{align*}
&\quad \mathrm{cost}(OPEN) - \mathrm{cost}(OPT) \\
&= \sumt \half \left( ||KK^\dagger y_t - (K^T)^\dagger s^*_t||^2 -  ||(KK^\dagger y_t - (K^T)^\dagger\hat{s}_t)||^2\right) \\
& + \half \left( ||\hat{y}_t - y_t||^2 - ||(I-KK^\dagger)(\hat{y}_t - y_t) ||^2\right) \\
&\le\sumt \half  ||\hat{y}_t - y_t||^2 - \half||(I-KK^\dagger)(\hat{y}_t - y_t) ||^2 \\
&= \sumt \half ||KK^\dagger(\hat{y}_t - y_t)||^2.
\end{align*}}

where the first inequality is because of the characterization of $s_t^*$ following \eqref{eqn: optimal_val}.
\end{proof}

\begin{proof}[of Lemma \ref{lem: cd-fhc-det}]
%With Lemma \ref{lem: cd-open-det}, Let $\mathrm{cost}(FHC^{(k)}_{\tau, \tau +
%w})$ be the cost incurred by $FHC^{(k)}$ from time steps $\tau$ to $\tau +w$.
%For all $k = 1, \cdots, w + 1$ and all $\tau \in \Omega_k$, we have:
%\begin{align}
%\label{eqn: cd-fhc-snapshot}
%&\mathrm{cost}(FHC_{\tau, \tau + w}) =  \sum_{t = \tau}^{\tau + w} \half ||y_t - Kx_t^{(k)}||^2 + \beta||x_t^{(k)} - x_{t-1}^{(k)}||_1 \notag \\
%\le & \sum_{t = \tau}^{\tau + w}\half ||(KK^\dagger)(y_{t|\tau} - y_t)||^2 + \half||y_t^* - Kx_t^*||^2  \notag\\
%& + \sum_{t = \tau + 1}^{\tau + w}\beta||x_t^* - x_{t-1}^*||_1 + \beta||x_{\tau-1}^* - x_{\tau-1}^{(k)}||_1 + \beta||x_{\tau-1}^* - x_{\tau}^*||_1 \notag\\
%%= &  \cost(OPT_{\tau, \tau+w}) +  \sum_{t = \tau}^{\tau + w}\half ||(KK^\dagger)(y_{t|\tau} - y_t)||^2 + \beta||x_{\tau-1}^* - x_{\tau-1}^{(k)}||_1
%\end{align}
%where $x^{(k)}_t, x_t^*$ are the actions chosen by $FHC^{(k)}$ and $OPT$ respectively. The inequality is by Lemma \ref{lem: cd-open-det} and the switching cost term satisfying the triangle inequality. Summing \eqref{eqn: cd-fhc-snapshot} for all the $\tau \in \Omega_k$ gives us Lemma \ref{lem: cd-fhc-det}. \end{proof}

The proof is a straightforward application of Lemma \ref{lem: cd-open-det}.  Summing the cost of $OPEN$ for all $\tau \in \Omega_k$ and noting that the switching cost term satisfying the triangle inequality gives us the desired result.
%Note that within this period, $FHC^{(k)}$ is $OPEN$ for $w+1$ steps with
%prediction $y_{t|\tau-1}$. Hence, f
\end{proof}

\subsection{Proof of Theorem \ref{thm: cd-afhc}}
We first define the sub-optimality of the open loop algorithm over expectation of the noise. $\ex[||(y_t - \hat{y_t})||_{KK^\dagger}^2]$ is the expectation of the projection of the prediction error $t+1$ time steps away onto the range space of $K$, given by:
{\small\begin{align*}
&\ex[||(y_t - \hat{y_t})||_{KK^\dagger}^2] = \ex ||\sum_{s=1}^t KK^\dagger(f(t-s)e(s))||^2 \\
= &\ex[\sum_{s_1=1}^t\sum_{s_2=1}^t e(s_1)^T f(t-s_1)^T (KK^\dagger)^T(KK^\dagger)f(t-s_2)e(s_2) ] \\
= &\mathrm{tr} (\sum_{s_1=1}^t\sum_{s_2=1}^t f(t-s_1)^T(KK^\dagger)(KK^\dagger) f(t-s_2)\ex[e(s_2)e(s_1)^T] ) \\
%&=\mathrm{tr} (\sum_{s_1=1}^t\sum_{s_2=1}^t  f(t-s_1)^T(KK^\dagger)f(t-s_2)R_e
%\delta_{s_1, s_2} ) \\
=&\sum_{s=0}^{t-1} \mathrm{tr}(f(s)^T (KK^\dagger)f(s)R_e),
\end{align*}}
where the last line is because $\ex[e(s_1)e(s_2)^T] = 0$ for all $s_1\ne s_2$, and $KK^\dagger K = K$. Note that this implies $||f_{t-1}||^2 = \sum_{s=0}^{t-1} \mathrm{tr}(f(s)^T f(s)R_e)$.
We now write the expected suboptimality of the open loop algorithm as
{\small\begin{align*}
&\ex[\mathrm{cost}(OPEN) - \mathrm{cost}(OPT)] \le \sumt \half \ex[||y_t - \hat{y_t}||_{KK^\dagger}^2] \\
%&= \half \sumt \sum_{s=0}^{t-1} \mathrm{tr}(f(s)^TKK^\dagger f(s)R_e) \\
&= \half \sum_{s=0}^{T-1} \sum_{t=s}^{T-1}  \mathrm{tr}(f(s)^TKK^\dagger f(s)R_e) \\
%&= \half \sum_{s=1}^t \sum_{t'=0}^{T-s}  \mathrm{tr}(f(t')^TKK^\dagger f(t')R_e) \\
%&= \half \sum_{t'=0}^{T-1} \sum_{s=1}^{T-t'}  \mathrm{tr}(f(t')^TKK^\dagger f(t')R_e) \\
&= \half \sum_{s=0}^{T-1} (T-s)  \mathrm{tr}(f(s)^TKK^\dagger f(s)R_e)
= F(T-1)
\end{align*}}
where the first equality is by rearranging the summation.

Now we take expectation of the expression we have in Lemma \ref{lem: cd-afhc-det}. Taking expectation of the second penalty term (prediction error term), we have:
\begin{align*}
&\frac{1}{w+1}\sum_{k=0}^{w}\sum_{\tau\in\Omega_k} \sum_{t=\tau}^{\tau+w}\mathbb{E}\half||(KK^\dagger)(y_t - y_{t|\tau-1})||^2 \\
%=&\frac{1}{2(w+1)}\sum_{k=0}^{w}\sum_{\tau\in\Omega_k} \sum_{t=\tau}^{\tau+w}||f_{t-\tau}||^2\\
=&\frac{1}{2(w+1)}\sum_{k=0}^{w}\sum_{\tau\in\Omega_k} F(w)
%=&\frac{T}{2(w+1)}\sum_{s=0}^{w - 1}(w - s)||f(s)R_e^{1/2} ||^2_{KK^{\dagger}, F} %\\
= \frac{T}{2(w+1)}F(w)
\end{align*}

We now need to bound the first penalty term (switching cost term). By taking the subgradient with respect to $x_t$ and by optimality we have  $\forall t=1, \ldots, T$
\begin{align*}
 & 0 \in K^T(Kx^*_t - y_t) + \beta\partial ||(x^*_t - x^*_{t-1})||_1 + \beta\partial ||(x^*_{t+1} - x^*_{t})||_1 \\
\Rightarrow & x^*_t \in [(K^TK)^{-1}(K^Ty_t-2\beta\mathbbm{1}), (K^TK)^{-1}(K^Ty_t+2\beta\mathbbm{1})]
\end{align*}
where the implication is because the sub-gradient of a 1-norm function $||\cdot||_1$ is between $-\mathbbm{1}$ to $\mathbbm{1}$.

Similarly, since $x^{(k)}_{\tau - 1}$ is the last action taken over a $FHC$ horizon, we have that for all $\tau \in \Omega_k$,
{\small\begin{align*}
x^{(k)}_{\tau-1} \in [(K^TK)^{-1}(K^Ty_{\tau-1|\tau-w - 2} - \beta\mathbbm{1}),
\\ (K^TK)^{-1}(K^Ty_{\tau-1|\tau-w - 2} + \beta\mathbbm{1})]
\end{align*}}

Taking expectation of one of the switching cost term and upper bounding with triangle inequality:
{\small\begin{align}
&\mathbb{E} \left|\left|(x^*_{\tau-1} - x^{(k)}_{\tau-1})\right|\right|_1 \notag \\
%\le & \frac{1}{w+1}\sum_{k=0}^{w}\sum_{\tau \in
%\Omega_k}(\mathbb{E}\left\beta|(K^TK)^{-1}K^Ty_{\tau-1} -
%(K^TK)^{-1}K^Ty_{\tau-1|\tau-1-w}\right| \\
%& + 3|\beta^T(K^TK)^{-1}\beta|)\\
\le &||K^{\dagger}||_1\mathbb{E}||y_{\tau-1} -y_{\tau-1|\tau-2-w}||_1 + 3\beta||(K^TK)^{-1}\mathbbm{1}||_1\notag \\
%&\le \frac{T}{w + 1}
%\left(\beta||\mathbbm{1}^TK^{\dagger}||_\infty\sqrt{\sum_{s=0}^{w}\mathrm{tr}(f(t')^TKK^\dagger f(t')R_e)} + 3\beta^2|\mathbbm{1}^T(K^TK)^{-1}\mathbbm{1}|\right) \\
\label{eqn: switch-one-term}
\le & ||K^{\dagger}||_1||f_w|| + 3\beta||(K^TK)^{-1}\mathbbm{1}||_1
\end{align}}
where the first inequality is by the definition of induced norm, the second inequality is due to concavity of the square-root function and Jensen's inequality.
%Hence we have
%\begin{align*} &\frac{1}{w+1}\sum_{k=0}^{w}\sum_{\tau\in\Omega_k}\ex \beta ||x^*_{\tau-1} - x^{(k)}_{\tau-1}||_1 \\
%\le &\frac{T}{w+1}\beta||K^{\dagger}||_1||f_w|| + 3\beta^2||(K^TK)^{-1}\mathbbm{1}||
%\end{align*}
Summing \eqref{eqn: switch-one-term} over $k$ and $\tau$, we have the expectation of the switching cost term.  Adding the expectation of both penalty terms (loss due to prediction error and loss due to switching cost) together, we get the desired result.

\subsection{Proof of Lemma \ref{lem: regret-open}}
We first characterize $\cost(STA)$:
{\small\begin{align*}
\cost(STA) = \min_x \half\sum^T_{t=1}||y_t - Kx||^2_{2} + \beta\mathbbm{1}^Tx
\end{align*}}

By first order conditions, we have the optimal static solution $x = K^\dagger \bar{y} - \frac{\beta}{T}(K^TK)^{-1}\mathbbm{1}$.
Substituting this to $\cost(STA)$ and simplifying, we have:
{\small\begin{align*}
cost(STA)  %&\half\sum^T_{t=1}||y_t - K(K^{\dagger}\bar{y} - \frac{\beta}{T}(K^TK)^{-1}\mathbbm{1}||^2_{2} \\
%&+ \beta\mathbbm{1}^T(K^{\dagger}\bar{y} - \frac{\beta}{T}(K^TK)^{-1}\mathbbm{1}) \\
%&= \half\sum^T_{t=1}\left(||\lmd_t - (KK^{\dagger})\bar{\lmd}||^2_{2} + \langle \lmd_t - (KK^{\dagger})\bar{\lmd}, \frac{1}{T}K(K^TK)^{-1}\beta \rangle\right) + \frac{||K(K^TK)^{-1}\beta||^2_{2}}{2T} + \beta^T(K^{\dagger}\bar{\lmd} - \frac{1}{T}(K^TK)^{-1}\beta) \\
%&= \half\sum^T_{t=1}\left(||\lmd_t - (KK^{\dagger})\bar{\lmd}||^2_{2}\right) - \frac{\beta^T(K^TK)^{-1}\beta}{2T}+ \beta^TK^{\dagger}\bar{\lmd}  + \langle (I - KK^{\dagger})\bar{\lmd}, K(K^TK)^{-1}\beta \rangle \\
%&= \half\sum^T_{t=1}\left(||\lmd_t - (KK^{\dagger})\bar{\lmd}||^2_{2}\right) - \frac{\beta^T(K^TK)^{-1}\beta}{2T}+ \beta^TK^{\dagger}\bar{\lmd} \\
= &\half\sum^T_{t=1}\left(||(I-KK^{\dagger})y_t||^2_{2} + ||KK^{\dagger}(y_t - \bar{y})||^2_{2} \right) \\
&- \frac{\beta^2\mathbbm{1}^T(K^TK)^{-1}\mathbbm{1}}{2T}+ \beta\mathbbm{1}^TK^{\dagger}\bar{y}
\end{align*}}

%From the proof of Lemma \ref{lem: cd-fhc-det}, we know:
%\begin{align*}
%&cost(OPT) = \sumt \half ||y_t ||^2 - \half||KK^\dagger y_t - (K^T)^\dagger s^*_t ||^2 \\
%&= \half\sum^T_{t=1}(||(I-KK^{\dagger})y_t||^2_{2} + ||(KK^{\dagger})y_t||^2_{2} \\
%&- ||KK^{\dagger}y_t - (K^T)^{\dagger}s^*_t||^2_{2})
%\end{align*}
Let $C = \frac{\beta^2\mathbbm{1}^T(K^TK)^{-1}\mathbbm{1}}{2T}$. Subtracting $\cost(OPT)$ in \eqref{eqn: optimal_val} from the above, we have $\mathrm{cost}(STA) - \mathrm{cost}(OPT)$ equals:
{\small\begin{align*}
&\half\sum^T_{t=1}(||KK^{\dagger}(y_t - \bar{y})||^2_{2} - ||KK^{\dagger}y_t||^2_{2} + ||KK^{\dagger}y_t - (K^T)^{\dagger}s^*_t||^2_{2}) \\
&+ \beta\mathbbm{1}^TK^{\dagger}\bar{y}  - C \\
%= &\half\sum^T_{t=1}\left(||KK^{\dagger}(y_t - \bar{y}) - (K^T)^{\dagger}s^*_t||^2_{2} - 2\langle KK^{\dagger}\bar{y}, (K^T)^{\dagger}s^*_t\rangle\right) \\
%&+ \beta\mathbbm{1}^TK^{\dagger}\bar{y}  - C  \\
%= &\half\sum^T_{t=1}\left(||KK^{\dagger}(y_t - \bar{y}) - (K^T)^{\dagger}s^*_t||^2_{2}\right) - \langle K^{\dagger}\bar{y}, \sum^T_{t=1}s^*_t\rangle \\
%&+ \beta\mathbbm{1}^TK^{\dagger}\bar{y}  - C  \\
= &\half\sum^T_{t=1}\left(||KK^{\dagger}(y_t - \bar{y}) - (K^T)^{\dagger}s^*_t||^2_{2}\right) + \langle K^{\dagger}\bar{y}, \beta\mathbbm{1} - \lambda_1\rangle - C \\
\ge &\half\sum^T_{t=1}\left(||KK^{\dagger}(y_t - \bar{y}) - (K^T)^{\dagger}s^*_t||^2_{2}\right) - C 
\end{align*}}

The first equality is by expanding the square terms and noting $s_t = \lambda_t - \lambda_{t+1}$. The last inequality is because  $-\beta\mathbbm{1} \le \lambda_t \le \beta\mathbbm{1}$ and  $\beta\mathbbm{1}^TK^{\dagger}\bar{y}$ being positive by assumption that the optimal static solution is positive. Now we bound the first term of the inequality above:
{\small\begin{align*}
&\half\sum^T_{t=1}\left(||KK^{\dagger}(y_t - \bar{y}) - (K^T)^{\dagger}s^*_t||^2_{2}\right) \\
\ge &\half\sum^T_{t=1}\left(||KK^{\dagger}(y_t - \bar{y})||^2\right) - \sum^T_{t=1}\langle KK^{\dagger}(y_t - \bar{y}), (K^T)^{\dagger}s^*_t \rangle \\
\ge &\small\half\sum^T_{t=1}\left(||KK^{\dagger}(y_t - \bar{y})||^2\right) -
2\beta\sum^T_{t=1}||KK^{\dagger}(y_t - \bar{y})||\,||(K^T)^{\dagger}\mathbbm{1}|| \\
\ge &\half\sum^T_{t=1}\left(||KK^{\dagger}(y_t - \bar{y})||^2\right) - 2B\sqrt{T\sum^T_{t=1}||(KK^{\dagger})(y_t - \bar{y})||^2} 
%\ge &\half\left(\sqrt{\sum^T_{t=1}\left(||KK^{\dagger}(y_t - \bar{y})||^2\right)} - 2B\sqrt{T}\right)^2 - 2B^2T
\end{align*}}
where $B = \beta||(K^T)^{\dagger}\mathbbm{1}||_2$.

By subtracting $C$ from the expression above and completing the sqaure, we have the desired result.
%We can now  lower bound the suboptimality of $STA$:
%\begin{align*}
%\mathrm{cost}(STA) - &\mathrm{cost}(OPT) \ge \half\sum^T_{t=1}\left(||KK^{\dagger}(y_t - \bar{y})||^2\right) \\
%-& 2B \sum^T_{t=1}||(KK^{\dagger})(y_t - \bar{y})||_2 - C.
%\end{align*}

\subsection{Proof of Theorem \ref{thm: regret-afhc}}
 Using the results of Lemma \ref{lem: regret-open}, taking expectation and
 applying Jensen's inequality, we have:
%\begin{align*}
%&\ex_e \cost(STA) - \cost(OPT) \\
%\ge &\half\left(\ex_e\sqrt{\sum^T_{t=1}\left(||KK^{\dagger}(y_t - \bar{y})||^2\right)} - 2B\sqrt{T}\right)^2 - 2B^2T - C.
%\end{align*}
{\small
\begin{align*}
&\ex_e \left[\cost(STA) - \cost(OPT) \right]\\
\ge & \ex_e \big[\half\sum^T_{t=1}||KK^{\dagger}(y_t - \bar{y})||^2 \!\!- 2B\sqrt{T\sum^T_{t=1}||(KK^{\dagger})(y_t - \bar{y})||^2} - C \big]\\
%\ge & \ex_e \half\sum^T_{t=1}\left(||KK^{\dagger}(y_t - \bar{y})||^2\right) - 2B\sqrt{T \cdot \ex_e \sum^T_{t=1}||(KK^{\dagger})(y_t - \bar{y})||^2} - C \\
 \ge &\half\left(\sqrt{\ex_e \sum^T_{t=1}\left(||KK^{\dagger}(y_t - \bar{y})||^2\right)} - 2B\sqrt{T}\right)^2 - 2B^2T - C.
\end{align*}}
%where the second inequality is due to Jensen's inequality.
Hence by Theorem \ref{thm: cd-afhc}, the regret of $AFHC$ is
{\small\begin{align*}
%&\small\sup_{\hat{y}}\ex_e [\mathrm{cost}(AFHC) - \mathrm{cost}(STA)] \nonumber \\
%=
&\small\sup_{\hat{y}}\big(\ex_e [\mathrm{cost}(AFHC) \!\!-\!\! \mathrm{cost}(OPT)
+ \mathrm{cost}(OPT) \!\!-\!\!  \mathrm{cost}(STA)]\big) \nonumber \\
&\le  VT \!+ 2B^2T \!+ C  
%\frac{T}{w + 1}
%\left(\beta||K^{\dagger}||_1||f_w|| + 3\beta^2||(K^TK)^{-1}\mathbbm{1}|| +
%\half F(w) \right) \nonumber  \\
\!- \half\inf_{\hat{y}} \!\! \left(\!\! \sqrt{\ex_e \!\sum^T_{t=1}||(y_t -
\bar{y})||_{KK^{\dagger}}^2} - 2B\sqrt{T}\!\right)^{\!\!\!2}\!\!.
%\label{eq:AFHC_regret_bound}
\end{align*}}
Let $ S(T) =  \ex_e \sum^T_{t=1}||y_t - \bar{y}||_{KK^{\dagger}}^2$.
By the above, to prove $AFHC$ has sublinear regret, it is sufficient that
%\[
%    V =
%\frac{1}{w+1}(\beta||K^{\dagger}||_1||f_w|| +
%3\beta^2||(K^TK)^{-1}\mathbbm{1}|| + \half F(w))
%\]
%and
%\[
%    S(T) =  \ex_e \sum^T_{t=1}||y_t - \bar{y}||_{KK^{\dagger}}^2.
%\]

%The theorem is now established as follows. If $\sum^T_{t=1}||(y_t -
%\bar{y})||_{KK^{\dagger}}^2 \in \omega(T)$, then the negative term
%\[
%\half\left(\sqrt{\sum^T_{t=1}\left(||(y_t -
%\bar{y})||_{KK^{\dagger}}^2\right)} - 2B\sqrt{T}\right)^2
%\]
%is $\in \omega(T)$. Since every other term in~\eqref{eq:AFHC_regret_bound}
%is in $O(T)$,
%the expectation of regret over noise is negative for sufficiently large $T$.
%
%To prove the alternate form, Theorem~\ref{thm: regret-afhc-lachlan},
{\small\begin{equation}
    VT + 2B^2T - \frac{1}{2}\inf_{\hat{y}}(\sqrt{S(T)} - 2B\sqrt{T})^2 < g(T)
    \label{eqn: variance_proj_cond}
\end{equation}}
for some sublinear $g(T)$. By the hypothesis of Theorem~\ref{thm: regret-afhc}, we have $\inf_{\hat{y}}S(T) \ge (8A + 16B^2)T$. \\
Then, $S(T) \ge (\sqrt{2VT + 4B^2T} + 2B\sqrt{T})^2$,
%\begin{align*}
%S(T)
%%&\ge (2\sqrt{2VT + 4B^2T})^2 \\
%& \ge (\sqrt{2VT + 4B^2T} + 2B\sqrt{T})^2
%\end{align*}
and \eqref{eqn: variance_proj_cond} holds since $VT + 2B^2T - \frac{1}{2}\inf_{\hat{y}}(\sqrt{S(T)} - 2B\sqrt{T})^2 \le VT + 2B^2T - \frac{1}{2} \left(\sqrt{2VT + 4B^2T}\right)^2 = 0$.

\section{Proofs for Section~\ref{SEC:CONC}}
\subsection{Proof of Lemma \ref{lemma: g1-characterize}}
By the triangle inequality, we have 
{\small\begin{align*}
g_1 = &\frac{1}{w + 1}\sum_{k=0}^{w}\sum_{\tau \in \Omega_k}\beta |x^*_{\tau-1} - x^{(k)}_{\tau -1}| \\
 \le &\frac{1}{w + 1}\sum_{k=0}^{w}\sum_{\tau \in \Omega_k}\beta
 \Big( |x^*_{\tau-1} - y_{\tau-1}| + |y_{\tau-1} - y_{\tau-1 | \tau  - w - 2}| \\
& + |y_{\tau - 1 | \tau - w -2} - x^{(k)}_{\tau-1}|\Big). 
\end{align*}}
By first order optimality condition, we have $x^*_{\tau - 1 } \in \{y_{\tau -1}
- 2\beta, y_{\tau - 1}+ 2\beta\}$, and $x^{(k)}_{\tau - 1} \in
\{y_{\tau-1|\tau-w-2} - \beta , y_{\tau-1|\tau-w-2} + \beta \}.$ Hence, by the prediction model, 
{\small\begin{align*}
g_1 &\le \frac{3\beta^2 T }{w+1}
       + \frac{\beta}{w+1}
	    \sum_{k=0}^w\sum_{\tau \in \Omega_k}
	    \left| \sum_{s=1\vee(\tau-w-2)}^{\tau-1} f(\tau-1-s)e(s)
	    \right| 
%&= \frac{3\beta^2 T }{w+1} + \frac{\beta}{w+1} \sum_{k=0}^w ||B_k e||_1.
\end{align*}}
\subsection{Proof of Lemma \ref{lemma: g1-ex-bound}}
Note that by Lemma \ref{lemma: g1-characterize}, we have 
{\small
\begin{align*}
\ex g'_1(e) &\le  \frac{\beta}{w+1} \sum_{k=0}^w\sum_{\tau \in \Omega_k} \ex
\left| \sum_{s=1\vee(\tau-w-2)}^{\tau-1} f(\tau-1-s)e(s) \right|  \\
%& \le \frac{\beta}{w+1} \sum_{k=0}^w\sum_{\tau \in \Omega_k} \sqrt{  \ex
%\left(\sum_{s=1\vee(\tau-w-2)}^{\tau-1}  f(\tau-1-s)e(s) \right)^2 } \\
& \le \frac{\beta}{w+1} \sum_{k=0}^w\sum_{\tau \in \Omega_k} \sqrt{  \sigma^2
\sum_{s=0}^w f^2(s)}  = \frac{\beta T}{w+1}||f_w||^2.
\end{align*}}
where the second inequality is by Jensen's inequality and taking expectation. 

\subsection{Proof of Lemma \ref{lemma: g2-characterization}}
By definition of $g_2$ and unraveling the prediction model, we have 
{\small\begin{align*}
g_2 &= \frac{1}{w + 1} \sum_{k=0}^w\sum_{\tau \in
\Omega_k}\sum_{t=\tau}^{\tau+w}\half(y_t - y_{t|\tau-1})^2 \\
&= \frac{1}{w + 1} \sum_{k=0}^w\sum_{\tau \in \Omega_k}\sum_{t=\tau}^{\tau+w}\half(\sum_{s=\tau}^{t} f(t-s)e(s) )^2 . 
\end{align*} }
Writing it in matrix form, it is not hard to see that 
\[ g_2 = \frac{1}{w+1}\sum_{k=0}^w \frac{1}{2}||A_k e||^2, \]
where 
$ A_k$ has the block diagonal structure given by
%{\scriptsize \begin{equation}
%A_k = 
%\left(
% \begin{array}{ccccc}
%   A_k^1\\
%    & A_k^2 & & \\
%    & & \ddots\\
%    &  & & A_k^2\\
%    & & & & A_k^3
% \end{array}
%\right),
%\label{def: amatrix}
%\end{equation}
%}
\begin{equation}\label{def: amatrix}
    A_k = \mbox{diag}(A_k^1, A_k^2, \dots, A_k^2, A_k^3),\footnote{\small The submatrix $A_k^2$ is repeated $\left \lfloor \frac{T-k+1}{\omega+1}\right \rfloor$ times in $A_k$ for $k \ge 2$, and $\left \lfloor \frac{T-k-\omega}{\omega+1}\right\rfloor$ times for otherwise.}
\end{equation} 
and there are the types of submatrices in $A_k$ given by, for $i=1,2,3$:
{\scriptsize
\[A_k^i = \begin{pmatrix}
f(0)& 0& \ldots & 0 \\
f(1)& f(0)& \ldots& 0\\
\vdots& \vdots& \ddots& \vdots\\
f(v_i)& f(v_{i-1})& \ldots& f(0)
\end{pmatrix}, 
\]
}
%\[
%A_k^2 = \begin{pmatrix}
%f(0)& 0& \ldots & 0 \\
%f(1)& f(0)& \ldots& 0\\
%\vdots& \vdots& \ddots& \vdots\\
%f(w)& f(w-1)& \ldots& f(0)
%\end{pmatrix},
%\]
%}
%and 
%{\scriptsize
%\[
%A_k^3 = \begin{pmatrix}
%f(0)& 0& \ldots & 0 \\
%f(1)& f(0)& \ldots& 0\\
%\vdots& \vdots& \ddots& \vdots\\
%f(l)& f(l-1)& \ldots& f(0)
%\end{pmatrix},
%\]
%}
where $v_1 = k-2$ if $k\ge 2$ and $v_1 = k+w-1$ otherwise. $v_2 = w$, $v_3 =
(T-k+1) \mod (w+1)$. %k = k + w -1$ if $k<2$, and $l = (T-k+1) \mod (w+1)$. 
Note that in fact, the matrix $A_k^2$ is the same for all $k$.
Hence, we have 
{\small
\begin{align*}
g_2 &=  \half e^T (\frac{1}{w+1}\sum_{k=0}^w A_k^T A_k ) e = \half ||Ae||^2,
\end{align*}}
where we define $A$ to be such that $A^T A =
\frac{1}{w+1}\sum_{k=0}^w A_k^T A_k $, this can be done because the
right-hand side is positive semidefinite, since $A_k$ is lower triangular. The last equality is because all $A_l^2$ has the same structure.   Let $\lambda$ be the maximum eigenvalue of $AA^T$, which can be expressed by
{\small\begin{align*}
\lambda &= \max_{||x||=1} x^T AA^T  x \\
&= \frac{1}{w+1}\max_{||x||=1}\sum_{k=0}^w x^T  A_k A_k^T x \le \frac{1}{w +
1}\sum_{k=0}^w \lambda_k,\\
\end{align*}}
where $\lambda_k$ is the maximum eigenvalue of $A_kA_k^T$. Note that $A_k$ has a block diagonal structure, hence $A_kA_k^T$ also has block diagonal structure, and if we divide the vector $x = (x_1, x_2, \ldots, x_m)$ into sub-vectors of appropriate dimension, then by the block diagonal nature of $A_k A_k^T$, we have 
{\small\begin{align*}
x^T A_k A_k^T x = & x_1^T A_k^1 {A_k^1}^T x_1 +  x_2^T A_k^2 {A_k^2}^T x_2 + \ldots \\
&+ x_{m-1}^T A_k^2 {A_k^2}^T x_{m-1} + x_m^T A_k^3 {A_k^3}^T x_m.
\end{align*}}
Hence, if we denote the maximum eigenvalues of $\lambda_k^i$ as the maximum eigenvalue of the matrix $A_k^i {A_k^i}^T$, then we have 
{\small\begin{align*}
\lambda_k &= \max_x \frac{x^T A_kA_k^T x}{x^T x} \\
&= \max_{x_1, \ldots, x_m} \frac{x_1^T A_k^1 {A_k^1}^T x_1 + x_2^TA_k^2{A_k^2}^T x_2 + \ldots + x_m^TA_k^3{A_k^3}^T x_m}{x_1^Tx_1 + \ldots + x_m^T x_m}\\
&\le \max_{x_1, \ldots, x_m} \frac{\max(\lambda_k^1, \lambda_k^2, \lambda_k^3) \cdot ({x_1^Tx_1 + \ldots + x_m^T x_m}) }{{x_1^Tx_1 + \ldots + x_m^T x_m}} \\
&\le \max(\lambda_k^1, \lambda_k^2, \lambda_k^3),
\end{align*}}
where $\lambda_k^i$ is the maximum eigenvalue of $A_k^i$ for $i \in \{1,2,3\}$. As $A_k^i {A_k^i}^T$ are all positive semidefinite, we can bound the maximum eigenvalue by trace, and note that $A_k^1$ and $A_k^3$ are submatrix of $A_k^2$, we have  
{\small\begin{align*}
\lambda_k &\le \max(\lambda_k^1, \lambda_k^2, \lambda_k^3) \le \mbox{tr}(A^2_k
{A^2_k}^T) = \frac{1}{\sigma^2} F(w). 
\end{align*}}
%\niangjun{fill in the lower bound for the spectral radius...}

\subsection{Proof of Lemma \ref{lemma: self-bound}}
To prove the lemma, we use the following variant of Log-Sobolev inequality
\begin{lemma}[Theorem 3.2, \cite{ledoux1999}]\label{lemma: log-sobolev}
Let $f:\mathbb{R}^n\rightarrow\mathbb{R}$ be convex, and random variable $X$ be supported on $[-d/2, d/2]^n$, then
{\small\begin{align*}
&\quad \ex[\exp(f(X)) f(X) ] - \ex[\exp(f(X))] \log \ex[\exp(f(X))] \notag \\
\le &\frac{d^2}{2} \ex[\exp(f(X)) ||\nabla f(X)||^2].
%\label{eqn: log-sobolev}
\end{align*}}
\end{lemma}

%If $f$ is further ``self-bounded'', then its tail probability can be bounded as in the following lemma.
We will use Lemma~\ref{lemma: log-sobolev} to prove
Lemma~\ref{lemma: self-bound}.
Denote the moment generating function of $f(X)$ by 
	\[ m(\theta) := \ex e^{\theta f(X)}, \qquad \theta >0. \]
The function $\theta f:\mathbb{R}^n\rightarrow \mathbb{R}$ is convex, and therefore it follows from Lemma \ref{lemma: log-sobolev} that
	\begin{align*}
	& \ex\left[e^{\theta f}\theta f\right] - \ex\left[e^{\theta f}\right] \ln\ex\left[e^{\theta f}\right] \leq \frac{d^2}{2} \ex\left[e^{\theta f} ||\theta\nabla f||^2\right], \\
	& \theta m'(\theta) - m(\theta) \ln m(\theta) \leq \frac{1}{2}\theta^2d^2 \ex[e^{\theta f} ||\nabla f||^2].
	\end{align*}
By to the self-bounding property \eqref{self_bound},
	\begin{align*}
	\theta m'(\theta) - m(\theta) \ln m(\theta)
	&~\leq~ \frac{1}{2}\theta^2d^2 \ex[e^{\theta f(X)} (af(X)+b)] \\
	&~=~ \frac{1}{2}\theta^2d^2 \left[ am'(\theta) + bm(\theta) \right].
	\end{align*}
Since $m(\theta)>0$, dividing by $\theta^2 m(\theta)$ gives
\begin{equation}\label{eq:conc_de}
\frac{d}{d\theta}\left[\left(\frac{1}{\theta} - \frac{ad^2}{2}\right)\ln m(\theta)\right] \le \frac{bd^2}{2}.
\end{equation}
Since $m(0)=1$ and $m'(0)=\ex f(X)=0$, we have
	\begin{align*}
	\lim_{\theta\rightarrow0^+} \left(\frac{1}{\theta} - \frac{ad^2}{2}\right)\ln m(\theta)
	%&~=~ \lim_{\theta\rightarrow0^+} \frac{1}{\theta} \ln m(\theta)                   -\lim_{\theta\rightarrow0^+} \frac{ad^2}{2}\ln m(\theta) \\
	%&~=~ \lim_{\theta\rightarrow0^+} \frac{m'(\theta)}{m(\theta)}                    -\frac{ad^2}{2}\ln m(0) \\
	~=~ 0,
	\end{align*}
and therefore
integrating both sides of~\eqref{eq:conc_de} from 0 to $s$ gives
%\[
%\left.\left(\frac{1}{\theta} - \frac{ad^2}{2}\right)\ln m(\theta)~\right|_{\theta=0}^s \le \frac{1}{2} bd^2s
%\]
%for $s\geq0$.
\begin{equation}
\label{moment generating function}
\left(\frac{1}{s} - \frac{ad^2}{2}\right)\ln m(s) \le \frac{1}{2} bd^2s,
\end{equation}
for $s\geq0$. We can bound the tail probability $\prob\{f>t\}$ with the control \eqref{moment generating function} over the moment generating function $m(s)$.

In particular,
	\begin{align*}
	\prob\{f(X) > t\}
	&~=~ \prob\left\{ e^{sf(X)} > e^{st} \right\} ~\leq~ e^{-st} \ex \left[
	e^{sf(X)} \right] \\%\inf_{\theta>0} \exp\left\{-\theta t + \frac{bd^2\theta^2}{2-ad^2\theta}\right\},\\
	&~=~ \exp[-st + \ln m(s)] \\
	&~\leq~ \exp\left[-st + \frac{bd^2s^2}{2-asd^2}\right],
	\end{align*}
for $s \in [0, 2/(ad^2)]$. Choose $s = t/(bd^2 + ad^2 t/2)$ to get
\[
\mathbb{P}\{f(X) >t\} \le \exp\left(\frac{-t^2}{d^2(2b + at)}\right) .%\qedhere
\]

% 
%\begin{lemma}
%The loss due to prediction error can be characterized by
%\label{lemma g2-characterization}
%\end{lemma} 
%
%
%Note that 
%\begin{align*}
%g_2 &= \frac{1}{w+1}\sum_{l=0}^w \frac{1}{2}||A_l e||^2 \\
%&= \half e^T (\frac{1}{w+1}\sum_{l=0}^w A_l^T A_l ) e \\
%&= \half ||Ae||^2,
%\end{align*}
%
%From Theorem \ref{thm: fhc}, we have 
%\begin{align*}
%\prob(g_2-\ex g_2> t/2) &\le \exp\left(\frac{-t^2}{16\epsilon^2 \lambda (4\ex g_2 + t)}\right) \\
%&= \exp\left(\frac{-t^2}{16\epsilon^2 \lambda (2\sigma^2\frac{T}{w+1}F(w) + t)}\right). 
%\label{eqn: g2_afhc}
%\end{align*}
%where $\lambda$ is the maximum eigenvalue of $AA^T$, which is the same as the
%maximum eigenvalue of $A^TA = \frac{1}{w+1}\sum_{l=0}^w A_l^T A_l $, but the maximum eigenvalue of $A^T A$ can be expressed by 
%\begin{align*}
%\lambda &= \max_{||x||=1} x^T A^T A x \\
%&= \frac{1}{w+1}\max_{||x||=1}\sum_{l=0}^w x^T  A_l^T A_l x \\
%&\le \frac{1}{w+1}  \sum_{l=0}^w \mathrm{tr}(A_l^2) \\
%&= \frac{1}{w+1} (w+1) F(w) = F(w).
%\end{align*}

\end{document}